\documentclass{article}
\usepackage[T1]{fontenc}
\usepackage{iclr2027_conference,times}

\usepackage{amsmath,amsfonts,bm}
\usepackage{amsthm}
\newtheorem{definition}{Definition}
\newtheorem{lemma}{Lemma}
\newtheorem{theorem}{Theorem}

\def\eqref#1{equation~\ref{#1}}

\def\1{\mathbf{1}}

\DeclareMathAlphabet{\mathsfit}{\encodingdefault}{\sfdefault}{m}{sl}
\SetMathAlphabet{\mathsfit}{bold}{\encodingdefault}{\sfdefault}{bx}{n}

\newcommand{\E}{\mathbb{E}}

\newcommand{\KL}{D_{\mathrm{KL}}}

\usepackage{booktabs}
\usepackage{algorithm}
\usepackage{algpseudocode}
\usepackage{hyperref}
\usepackage{cleveref}
\usepackage{amssymb}
\usepackage{bbm}
\usepackage{amsthm}

\usepackage{booktabs}
\usepackage{placeins}
\usepackage{pgfplots}
\pgfplotsset{compat=1.17}
\usepgfplotslibrary{groupplots}
\pgfplotsset{compat=1.16}
\definecolor{cPlain}{HTML}{8C8C8C}
\definecolor{cFixed}{HTML}{9EC1E8}
\definecolor{cAdapt}{HTML}{1F5FBF}
\usepackage{wrapfig}
\usepackage{hyperref}
\usepackage{url}
\usepackage{tikz}
\usepackage{algorithm}
\usepackage{siunitx}
\usepackage{subcaption}
\usepackage{amsthm}
\newtheorem{corollary}{Corollary}
\usepackage{mathtools}
\usepackage[table]{xcolor}
\usepackage{booktabs}
\usepackage{siunitx}
\usepackage{thmtools}
\usepackage{thm-restate}

\usepackage{enumitem}
\usepackage{pgfplots}
\pgfplotsset{compat=1.18}
\usepackage{amsmath,amsthm} 
\newtheorem{proposition}{Proposition}
\newtheorem{remark}{Remark}
\newtheorem{example}{Example}
\title{Adaptive Multi-Token Prediction: \\ How far should we go?}

\author{Emile Anand, Mohammed Rehan Parwani, Xinyuan Cao \\
Georgia Institute of Technology\\
\texttt{\{eanand,mparwani3,xcao78\}@gatech.edu}}

\usetikzlibrary{arrows.meta, positioning}

\begin{document}

\iclrfinalcopy
\maketitle

\begin{abstract}
Multi-token prediction (MTP) can accelerate language-model decoding by drafting several tokens and verifying them in a single target-model pass. However, the best speculation horizon changes with context: longer drafts offer greater amortization when reliable, but can waste compute on suffixes discarded after a rejection. Unlike autoregressive drafting, where a stopping rule can decide token by token, MTP must fix the horizon before drafting, as all heads fire in the same forward pass. We therefore formulate horizon selection as an online control problem, using verification feedback to learn how far the draft can be trusted in each context, and introduce Adaptive Multi-Token Prediction (AMTP): a lightweight controller that requires no retraining and leaves the exact verifier unchanged. 
To guide horizon selection, we characterize the expected number of tokens committed at each horizon via \emph{acceptance survival}, the probability that a draft prefix is fully accepted. 
Before each round, AMTP predicts how far the draft is likely to survive and picks the horizon with the highest expected tokens per second, learning freely from every verification pass. Theoretically, we show that any online learner's prediction error translates into throughput regret against the best horizon for each context, and in turn into a bound on wall-clock completion time. With online gradient descent, this regret is sublinear. 
Across five draft–target pairings built on Llama-3.1-8B-Instruct with Medusa heads, Meta’s MTP models, and Qwen3-4B-Inst-2507-MTP, evaluated on AIME 2025 and GSM8K, AMTP decodes up to 4.1$\times$ faster than standard decoding and matches or outperforms the fixed window in 8 of 10 settings, with throughput gains of up to 10\%.
\end{abstract}

\section{Introduction}
Autoregressive decoding generates each new token through an evaluation of a target language model. Speculative decoding reduces this sequential bottleneck by using a cheaper draft mechanism to propose several tokens and evaluating them in parallel with the target model \citep{stern2018blockwise,leviathan2023fast,chen2023accelerating}. With an appropriate acceptance-and-correction rule, this procedure preserves the target distribution while committing multiple tokens in a single round \citep{leviathan2023fast,chen2023accelerating,xia2024unlocking}. Multi-token prediction (MTP) provides a natural proposal mechanism by predicting several future tokens from a shared representation \citep{gloeckle2024better,cai2024medusa,ankner2024hydra}, and have witnessed enormous success in practice: for instance, MTP modules are now shipped in frontier models \citep{deepseek2024v3}, and pretrained next-token models can be converted into MTP drafters without architectural changes \citep{samragh2025future,kirchenbauer2026mtp,zhao2026mtpd}. These methods create an inference-time decision that is distinct from learning a good draft model: \emph{how many tokens should be proposed before the next verification step?} We refer to this quantity as the \textit{speculation horizon}.\looseness=-1

The speculation horizon directly affects \emph{decoding efficiency}. A longer horizon amortizes verification over more committed tokens when the draft is reliable, but after an early rejection the discarded suffix adds latency and no output. Horizon selection therefore balances the expected gain from additional accepted tokens against the risk of speculating too far. We evaluate each horizon by its conditional throughput: the ratio of its expected number of committed tokens to its expected round latency. Even for a fixed draft-target pair, the throughput-maximizing horizon varies with context, because acceptance probabilities at each draft position depend on context. For MTP, this choice has a structure that isn't exploited in current methods \citep{anand2026smatattentionstructuredlongcontextsequence}: an autoregressive drafter can decide token by token whether to keep drafting, but MTP heads produce all draft tokens in a single forward pass, so the horizon must be fixed before drafting, from information available at that point.

Existing adaptive methods do not make this  decision without additional training. Most stop an autoregressive drafter token by token, using a trained classifier or acceptance head \citep{mamou2024dynamic,huang2024specdecpp}, entropy-based bounds \citep{agrawal2024adaedl,zhang2025svip}, or parallel drafting and verification \citep{liu2025pearl}. Others shape draft trees from confidence after drafting \citep{li2024eagle2,li2025eagle3,chen2026entmtp}, prune verification prefixes with predicted survival probabilities \citep{cheng2026dspark}, or train reinforcement-learning policies for draft depth \citep{zhang2026learning}. Concurrent work trains a router that selects the span length before decoding, but requires new prediction heads and verifies against greedy decoding only \citep{chintha2026cadence}. In fact, to our knowledge, AMTP is the first training-free method that selects the MTP horizon before drafting at every round. A complementary line of work uses verification feedback online to adapt the draft model \citep{liu2024online,qian2026drafts} or to select among drafters and hyperparameter configurations with regret guarantees against the best fixed choice \citep{liu2026notabandit,hou2025banditspec}. However, in this work, we compare against the stronger benchmark of the best horizon at each visited context.  We study how learning acceptance behavior from verification feedback affects horizon-selection performance, and, for a given target model and draft mechanism, we develop an online controller and bound its conditional throughput regret relative to the best horizon at each visited context in terms of cumulative prediction error.
 Section~\ref{sec:related-work} discusses related work in detail.

To relate prediction error to horizon-selection, we first identify the acceptance statistics that influence the expected number of committed tokens for each horizon. A draft token can be committed as an accepted proposal only if all prior proposals are also accepted. We characterize the joint acceptance event via \emph{acceptance survival}, the probability that a draft prefix is fully accepted, which further determines the expected number of committed tokens exactly. We express them in terms of \emph{conditional rejection hazards}, where the hazard at a position is its rejection probability given acceptance of all prior proposals. Our analyses holds for general causal proposals, including independent and correlated prediction heads \citep{cai2024medusa,ankner2024hydra}. Estimating the hazards requires distinguishing between proposed and surviving prefixes: verification provides draft-target discrepancies along sampled proposal prefixes, whereas the hazards average these discrepancies under the distribution of prefixes that survive verification. So, we derive a survival-weighted regression objective that targets these conditional hazards, and its excess population loss is the survival-weighted hazard-prediction error that enters our throughput-regret bound, mirroring decision-focused learning, where prediction losses are judged by the induced decisions \citep{elmachtoub2022smart}.

We use this objective to develop \emph{Adaptive Multi-Token Prediction} (AMTP), a training-free horizon controller equipped with an online adaptive learning algorithm. Before drafting, AMTP combines hazard-based forecasts of committed tokens with estimates of round latency to evaluate the throughput of each horizon. Short horizons, however, conceal verification feedback at later positions. To learn under this partial feedback, AMTP combines greedy throughput-based selection with randomized full-horizon exploration \citep{langford2007epoch} and uses inverse-probability weighting to obtain unbiased estimates of the survival-weighted regression loss. 

Our contributions are as follows:
\begin{itemize}
    \item \textbf{Acceptance survival.} We characterize horizon selection through accepted-prefix survival and round cost for general causal proposals, showing how survival determines the expected number of committed tokens for each horizon (\S\ref{sec:formulation}).
    \item \textbf{AMTP.} We introduce a training-free online controller that chooses the MTP horizon before each round, by maximizing predicted tokens per second, learning from partial feedback of short horizons via randomized exploration and inverse-probability weighting (\S\ref{sec:controller}).
    \item \textbf{Prediction error to wall-clock time.} We show that the prediction error of any online learner translates into throughput regret against the best horizon for each context, and in turn into a bound on expected completion time. With online gradient descent \citep{zinkevich2003online} under partial feedback, the regret is $O(T^{4/5}(1+P_T)^{1/5})$, which is sublinear (\S\ref{sec:guarantees}).
    \item \textbf{Experiments.} Across five draft--target pairings and two reasoning benchmarks, AMTP decodes up to 4.1$\times$ faster than standard decoding and matches or outperforms the fixed window in 8 of 10 settings, with throughput gains of up to 10\% (\S\ref{sec:experiments}).
\end{itemize}

\section{Related Work}

\label{sec:related-work}
 
\textbf{Speculative decoding and multi-token proposals.}
Speculative decoding drafts tokens with a cheap proposal and verifies them in parallel with the target model, using an acceptance-and-correction rule that leaves the target distribution unchanged \citep{leviathan2023fast,chen2023accelerating}. Most work uses a separate next-token draft model, or, as in EAGLE, a lightweight head that drafts autoregressively at the feature level \citep{li2024eagle}. Multi-token prediction (MTP) instead trains several future-token heads on a shared representation \citep{gloeckle2024better}; Medusa adds such heads to an existing model and verifies candidate continuations with tree attention \citep{cai2024medusa}, and recent work obtains MTP models from pretrained next-token models by self-distillation \citep{kirchenbauer2026mtp}. Joint multi-token prediction studies how to learn representations that support predicting several tokens at once \citep{ahn2025joint}. We keep the exact verifier and study a different question: given a fixed proposal, how many tokens should we draft in each round?
 
\textbf{Adaptive draft length for autoregressive drafters.}
A substantial line of work adapts the draft length of a next-token drafter. DISCO trains a classifier that decides after each drafted token whether to continue \citep{mamou2024dynamic}; SpecDec++ adds an acceptance-prediction head to the draft model and stops once the predicted probability of a rejection crosses a threshold \citep{huang2024specdecpp}; AdaEDL stops when an entropy-based lower bound on the acceptance probability falls below a threshold \citep{agrawal2024adaedl}; and SVIP derives a similar entropy-based stopping policy for long-form generation \citep{zhang2025svip}. EAGLE-2 instead grows a draft tree whose shape follows context-dependent confidence \citep{li2024eagle2}, and Learning to Draft trains reinforcement-learning policies for draft depth and verification size with throughput rewards \citep{zhang2026learning}.
These methods share a structural assumption: drafting is sequential, so the decision to stop can be made \emph{token by token}, after each drafted token is observed and before the next one is paid for. However, MTP removes this option. All heads fire in the same forward pass, so the cost of drafting is incurred up front and the horizon must be fixed \emph{before} any draft token is sampled. Truncating a complete MTP draft afterwards saves verification cost but not drafting cost, and choosing the cut point after inspecting sampled tokens can break exactness (Appendix~\ref{sec:setup-exactness}). AMTP is designed for this setting: it chooses the horizon from pre-draft information only, and our statistical and online-optimization analysis of horizon selection and acceptance survival applies to any causal proposal, including both independent and correlated heads.

\textbf{Online learning for speculative decoding.}
Verification feedback has also been used to adapt speculative decoding online. OnlineSpec updates an evolving draft model and relates its dynamic regret to accepted length at a fixed candidate length \citep{qian2026drafts}. Not-a-Bandit gives no-regret guarantees for selecting among draft models \citep{liu2026notabandit}. BanditSpec treats speculative-decoding hyperparameters, including the speculation length, as bandit arms and bounds regret against the best fixed configuration \citep{hou2025banditspec}. Our guarantee uses a stronger comparator, the best horizon \emph{for each context} along the realized decoding trajectory, and is learner-agnostic: any online regression method's prediction error transfers to throughput regret, with online gradient descent \citep{zinkevich2003online} as one instantiation. The broader strategy of translating prediction or optimization error into decision quality is related to perturbation-based model predictive control \citep{lin2022mpc} and online adaptive policy selection \citep{lin2023gaps,pmlr-v247-lin24a}, which rely on assumptions about the controlled dynamics, such as contractive perturbations, that we do not impose on the language-model prefixes.

\section{Speculation as a Survival Problem}

We first describe one round of speculation, and then show that a single sequence of rejection hazards determines the expected gain of every horizon.

\label{sec:formulation}

\subsection{One Round of Speculation: Draft, Verify, Commit}
\label{sec:setup-causal}
At round $t$, let $H_t=x_{1:n_t}$ be the committed prefix and let $\mathcal F_t$ contain all information available before the next horizon is chosen, including $H_t$, system state, and the controller parameters. For $1\le k\le K$, we define the target and proposal block laws
\begin{equation}
 P_t^{(k)}(y_{1:k})=\prod_{j=1}^k p_{t,j}(y_j| y_{<j}),
 \qquad
 Q_t^{(k)}(y_{1:k})=\prod_{j=1}^k q_{t,j}(y_j| y_{<j}),
 \label{eq:causal-laws}
\end{equation}
where conditioning on $H_t$ and the fixed proposal parameters $\phi$ is suppressed. We let $\smash{P_t\coloneqq P_t^{(K)}}$ and $\smash{Q_t\coloneqq Q_t^{(K)}}$. In this lens, independent MTP heads form the special case $\smash{q_{t,j}(\cdot| y_{<j})=q_{t,j}(\cdot| H_t)}$, where recurrent or otherwise correlated heads \citep{ankner2024hydra} keep the dependence on $y_{<j}$, and our analytic framework covers both of these cases.

\begin{figure}[t]
\centering
\includegraphics[width=\linewidth]{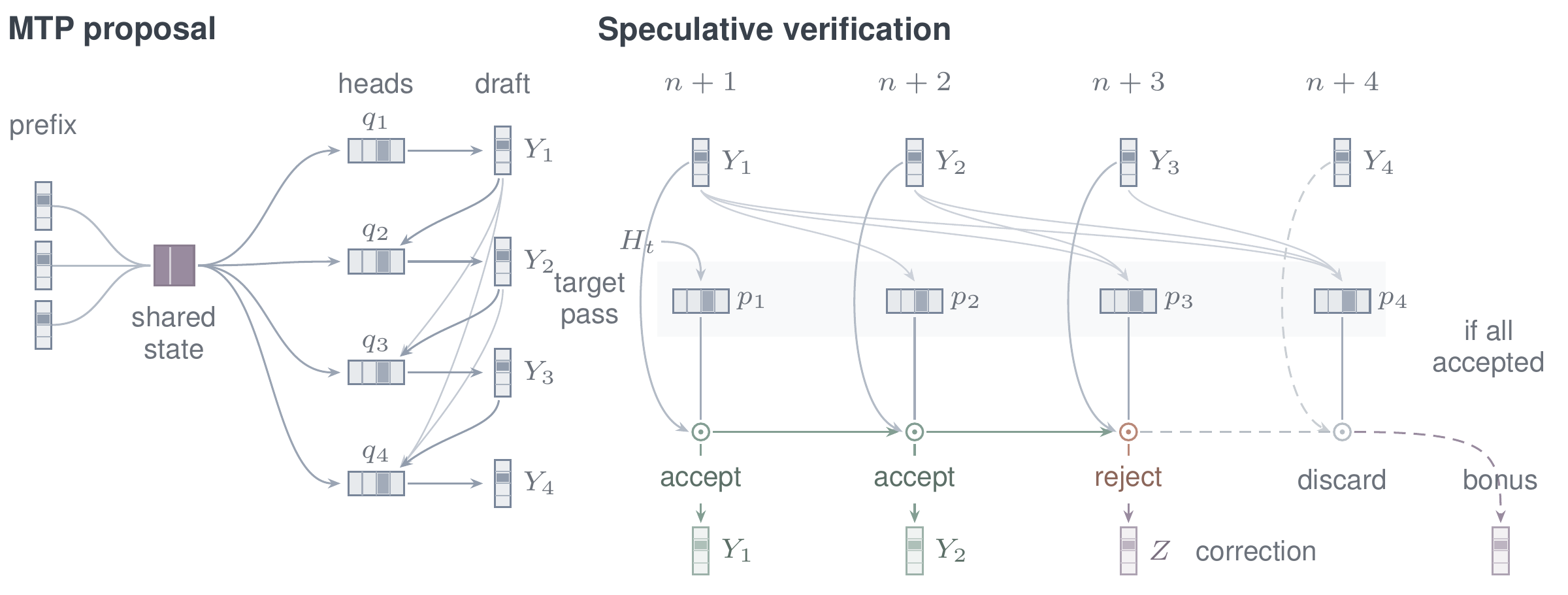}
\caption{\textbf{MTP with exact verification.} Proposal heads condition on earlier sampled tokens. One target pass scores $p_j(\cdot| H_t,Y_{<j})$, and circles denote prefix-ordered acceptance. Rejection emits a residual correction $Z$. The dashed bonus branch requires full acceptance.}
\label{fig:mtp-speculative}
\end{figure}

The horizon (or action) $k_t\in\{1,\dots,K\}$ is chosen \emph{before} the current draft is sampled. Moreover, every horizon truncates the same causal law $Q_t$, so changing $k$ does not change the distribution of earlier draft tokens: it only decides how many of them are drafted and sent to the verifier.
 A round proceeds as in Figure~\ref{fig:mtp-speculative}: the drafter samples $Y_{1:k}\sim Q_t^{(k)}$, and a single target pass scores all $k$ positions. Exact speculative verification \citep{leviathan2023fast,chen2023accelerating} then checks the draft from left to right, accepting $Y_j$ with probability
\begin{equation}\alpha_{t,j}\coloneqq \min\left\{1,\frac{p_{t,j}(Y_j| Y_{<j})}{q_{t,j}(Y_j| Y_{<j})}\right\},\end{equation}
where $q_{t,j}(Y_j| Y_{<j})>0$ on a sampled proposal. At the first rejection, the verifier draws a correction token from the normalized positive part of $p_{t,j}-q_{t,j}$. If all $k$ proposals are accepted, it draws a bonus token from the target conditional at the end of the draft. So, a round with horizon $k$ commits $L_t(k)=1+A_t(k)$ tokens, where $A_t(k)$ is the length of the accepted draft prefix, and the committed tokens follow the target law exactly \citep{leviathan2023fast,chen2023accelerating}. We analyze a fixed output-token budget where truncating the final output preserves its target law.\looseness=-1

\textbf{Exactness requires choosing the horizon before drafting.}
\label{sec:setup-exactness}
Exactness continues to hold when $k_t$ is chosen adaptively, provided it depends only on $\mathcal{F}_t$ and on randomness independent of the current draft. Earlier acceptance decisions depend only on earlier proposals and independent coins, so the verifier reaches position $j$ with the proposal still distributed as $q_{t,j}(\cdot | Y_{<j})$, and the standard one-step argument applies at every position and every round, however poorly the horizon is chosen. However, choosing the horizon \emph{after} seeing the draft can break this, even for a perfect drafter. To see this, let $p = q$ be independent fair bits, draw $Y_1, Y_2$, and keep $k = 2$ if $Y_2 = 1$ and $k = 1$ otherwise. Every drafted token is accepted, so the second emitted token is $Y_2 = 1$ on the long branch and a new fair bonus token on the short one. It equals $1$ with probability $3/4$ rather than $1/2$. AMTP therefore fixes $k_t$ before sampling the draft. Sampling transformations such as temperature or top-$p$ truncation must be built into the $p$ and $q$ the verifier uses.

Moreover, how likely a draft token is to pass depends on how far the proposal is from the target at that position. For a proposed prefix $u=y_{<j}$, we can measure this by the local discrepancy
\begin{equation}
 e_{t,j}(u)\coloneqq \mathrm{TV}\bigl(p_{t,j}(\cdot| u),q_{t,j}(\cdot| u)\bigr)
 =\tfrac12\|p_{t,j}(\cdot| u)-q_{t,j}(\cdot| u)\|_1,
 \label{eq:local-tv}
\end{equation}
which is exactly the probability that position $j$ is rejected given that the verifier reaches it along prefix $u$:  equivalently, $\mathbb E[\alpha_{t,j}|Y_{<j}=u,\mathcal F_t]=1-e_{t,j}(u)$.

Importantly, accurate token marginals do not guarantee an accurate multi-token proposal. For instance (as in Appendix~\ref{app:motivation-proofs}), independent MTP heads can fit every position perfectly while missing dependencies among future tokens. The block KL divergence splits into a marginal fitting error and a dependence term that remains even for exact marginals. In a simple repeated-bit example, the block mismatch approaches its maximum while the expected number of committed tokens per round stays below three for every horizon, so longer windows add a verification cost with little additional token gain. The horizon should therefore be chosen from how long draft prefixes actually survive verification, and from what each round costs, rather than from marginal prediction accuracy alone.\looseness=-1

\subsection{Survival and the Rejection Hazard}
\label{sec:setup-hazard}

We now make the the notion of ``how long draft prefixes survive'' precise, and formally define the acceptance survival and rejection hazard. For a virtual full draft $Y\sim Q_t$, we define the prefix acceptance weight via $W_{t,0}=1$ and $W_{t,j}=\prod_{r=1}^j \alpha_{t,r}$, which is the probability that the verifier accepts all of $Y_{1:j}$ given the sampled tokens.
\begin{definition}[Acceptance survival] The acceptance survival is the probability of accepting at least $j$ draft tokens, conditional on $\mathcal F_t$, and is given by
$s_{t,j}=\mathbb E_{Q_t}[W_{t,j}]$ with $s_{t,0}=1$.
\end{definition}

\begin{definition}[Rejection hazard] The rejection hazard at position $j$ is the probability of rejection there, given that the preceding $j-1$ proposals were accepted, and is given (conditional on $\mathcal F_t$) by \begin{equation}h_{t,j}=\frac{\E_{Q_t}[W_{t,j-1} e_{t,j}(Y_{<j})]}{s_{t,j-1}}.\label{eq:mtp-hazard}\end{equation}
\end{definition}
When $s_{t,j-1}=0$, we set $h_{t,j}=0$, as later survival probabilities are zero regardless of this convention. Since $\E[\alpha_{t,j}|Y_{<j},\mathcal{F}_t]=1-e_{t,j}(Y_{<j})$, the survival and expected gain follow from the hazard:\looseness=-1
\begin{equation}
 s_{t,j}=\prod_{r=1}^j(1-h_{t,r}),
 \qquad
 G_t(k):=\E[L_t(k)|\mathcal{F}_t]=1+\sum_{j=1}^k s_{t,j},
 \label{eq:mtp-survival}
\end{equation}
where quantities indexed by an action refer to the potential outcome of choosing that horizon, even if the policy assigns it zero probability. To justify why hazards are the right prediction target, we point to Equation \ref{eq:mtp-survival}: a single vector $(h_{t,1}\dots,h_{t,K})$ determines the expected gain of all $K$ horizons at once, so a controller that predicts it can score every horizon before drafting. Indeed, the identities need no independence between draft positions, so they hold for correlated heads as well as independent ones.

\textbf{Why the hazard is survival-weighted.} The numerator of Equation \ref{eq:mtp-hazard} weights each prefix by its likelihood of surviving to position $j$. Replacing $h_{t,j}$ with the unweighted mean discrepancy $\E_{Q_t}[e_{t,j}(Y_{<j})]$ is generally incorrect since prefixes that survive verification are distributed differently from arbitrary draft prefixes. For instance, consider a binary vocabulary $\{a,b\}$. The first head proposes $a$ or $b$ with probability $1/2$ each, while the target places mass $3/4$ on $a$, so $\alpha_{t,1}=1$ on $a$, $\alpha_{t,2}=1/2$ on $b$, and $s_{t,1}=3/4$. Then, suppose the second head matches the target exactly after $a$ but has disjoint support from it after $b$, so $e_{t,2}(a)=0$ and $e_{t,2}(b)=1$. The unweighted mean discrepancy at position $2$ is $1/2$, but $\smash{h_{t,2}=\frac{1/4}{3/4} = 1/3}$, since a draft that starts with the poorly matched token $b$ is less likely to reach position $2$ at all. Here, the true gain is $\smash{G_2(2)=1+\tfrac34 +\tfrac12=2.25}$, and plugging this in the unweighted mean gives $2.125$. Section \ref{sec:controller} shows how to build this change of prefix distribution into the learning objective.

\section{Online Horizon Control}

\label{sec:controller}

Each round of speculative decoding poses a one-step control problem: before drafting, choose the horizon $k_t$ that maximizes throughput. If the rejection hazards were known, Section \ref{sec:formulation} would solve this problem exactly, since Equation \ref{eq:mtp-survival} gives the expected gain of every horizon. However, since they are not known,  AMTP uses certainty equivalence: it plugs in predicted hazards, acts on them, and updates the predictor via online gradient descent (OGD) from the feedback that verification provides. Algorithm \ref{alg:amtp} gives the full loop and Figure \ref{fig:ogd-controller} illustrates it.

\begin{algorithm}[hbt!]
\caption{Adaptive Multi-Token Prediction}
\label{alg:amtp}
\begin{algorithmic}[1]
\Require Maximum horizon $K$; parameter set $\Theta$ and $\theta_1 \in \Theta$; exploration rate $\varepsilon \in (0,1]$; step sizes $\eta_t$; initial cost model $\widehat C_1$
\For{each speculative round $t = 1, 2, \dots$}
  \State Observe pre-draft features $z_t$; predict hazards $\hat h_{\theta_t,t,j}$ and costs $\widehat C_t(j)$ for $j = 1,\dots,K$
  \State Set $\kappa_t \gets \arg\max_k \widehat G_t(k)/\widehat C_t(k)$ using~\eqref{eq:plugin} \Comment{plug-in control, $O(K)$}
  \State With probability $\varepsilon$ set $k_t \gets K$; otherwise $k_t \gets \kappa_t$ \Comment{exploration}
  \State Draft $Y_{1:k_t} \sim Q_t^{(k_t)}$ and verify all $k_t$ positions in one target pass
  \State Commit the accepted prefix and one correction or bonus token
  \State For $j \le k_t$: record $e^{\mathrm{obs}}_{t,j} = e_{t,j}(Y_{<j})$, $W_{t,j-1} = \prod_{r<j} \alpha_{t,r}$ and $\pi_{t,j} = (1-\varepsilon)\mathbf{1}\{j \le \kappa_t\} + \varepsilon$
  \State Update $\theta_{t+1} \gets \Pi_\Theta\bigl(\theta_t - \eta_t \nabla \hat\ell_t(\theta_t)\bigr)$ with $\hat\ell_t$ from~\eqref{eq:mtp-population-loss} \Comment{OGD}
  \State Record the measured round cost $c_t$ and update the cost model $\widehat C_{t+1}$
\EndFor
\end{algorithmic}
\end{algorithm}

\textbf{The Control Problem.} Let $C_t(k)$ be the conditional expected cost of a round with horizon $k$, including drafting, verification, synchronization and all controller overhead, with $0 < c_{\min} \le C_t(k) \le c_{\max} < \infty$. We compare horizons at the same learner state and under the same update protocol, where the controller's objective is its throughput, and the optimal horizon at the current history is
\begin{equation}
\label{eq:throughput}
R_t(k) = \frac{G_t(k)}{C_t(k)},
\qquad
k_t^\star \in \arg\max_{1 \le k \le K} R_t(k).
\end{equation}
AMTP predicts the rejection hazards $\hat h_{\theta_t,t,j} \in [0,1]$ from pre-draft features $z_t$ and position $j$, and estimates rounds costs $\widehat C_t(k) > 0$, where the predictions are $\mathcal F_t$-measurable. It converts the hazard predictions into survival probabilities and expected token gains and selects a candidate horizon:\looseness=-1
\begin{equation}
\label{eq:plugin}
\hat s_{t,j} = \prod_{r=1}^{j} \bigl(1 - \hat h_{\theta_t,t,r}\bigr),
\qquad
\widehat G_t(k) = 1 + \sum_{j=1}^{k} \hat s_{t,j},
\qquad
\kappa_t \in \arg\max_{1 \le k \le K} \frac{\widehat G_t(k)}{\widehat C_t(k)},
\end{equation}
with a fixed tie-breaking rule. A running product and a running sum score all $K$ horizons in $O(K)$ time, and the controller does not assume that throughput is unimodal in $k$, and so the only part of Equation \ref{eq:plugin} that must be learned is the hazard forecaster $\theta_t$.

\subsection{Learning the hazards via Online Gradient Descent (OGD)}

OGD needs a loss whose error is reflective of the error attributed to the control problem. Following the change of prefix distribution in Section~\ref{sec:setup-hazard}, we use the survival-weighted population  loss
\begin{equation}
\ell_t(\theta) = \mathbb{E}_{Q_t}\left[\sum_{j=1}^{K} W_{t,j-1} \bigl(e_{t,j}(Y_{<j}) - \hat h_{\theta,t,j}\bigr)^2\right],\label{eq:mtp-population-loss}
\end{equation}
conditional on $\mathcal{F}_t$, with the proposal law and features held fixed. Let $\ell_t^{B}$ be the same expression with the true hazards $h_{t,j}$ in place of $\hat h_{\theta,t,j}$. Then, $\ell_t^B$ represents the minimum loss over unrestricted pre-draft hazard predictions. We then show the following calibration identity:
 
\begin{lemma}[Calibration identity]
\label{lem:mtp-calibration}
For every $\mathcal{F}_t$-measurable parameter $\theta$,
\begin{equation}
\mathcal{E}_t(\theta) := \ell_t(\theta) - \ell_t^{B} = \sum_{j=1}^{K} s_{t,j-1} \bigl(\hat h_{\theta,t,j} - h_{t,j}\bigr)^2.
\label{eq:mtp-calibration}
\end{equation}
\end{lemma}
 
Expanding the weighted square around $h_{t,j}$ yields the
identity, and the cross term vanishes by the hazard definition in (\ref{eq:mtp-hazard}); so, the loss is minimized by the true hazards, and its excess is a hazard error weighted by how often each position is reached: it is weighted less on positions that are less likely to survive verification. 
Section \ref{sec:guarantees} shows that this excess controls the throughput lost.

\begin{figure}[t]
\centering
\includegraphics[width=0.9\linewidth]{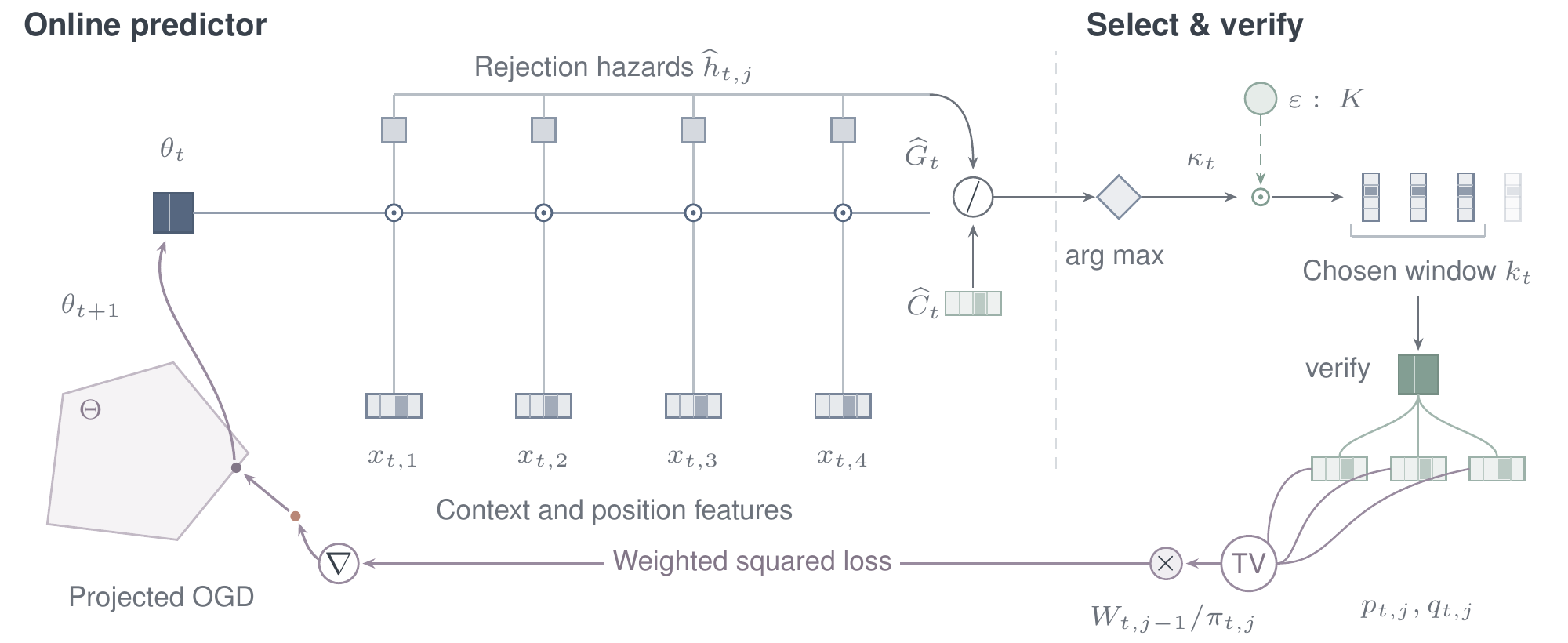}
\caption{\textbf{AMTP's online control loop.}  Before drafting, it picks the horizon with the highest predicted throughput. Verification feedback at every drafted position, including those after a rejection, enters the weighted loss in Equation \ref{eq:mtp-population-loss}, and an OGD step updates the hazard predictor.}
\label{fig:ogd-controller}
\end{figure}

\textbf{The ideal learner.}
If verification revealed $e^{\mathrm{obs}}_{t,j} = e_{t,j}(Y_{<j})$ and $W_{t,j-1}$ at all $K$ positions of a draft $Y \sim Q_t$, then the realized loss $\smash{\tilde\ell_t(\theta) = \sum_{j=1}^{K} W_{t,j-1} (e^{\mathrm{obs}}_{t,j} - \hat h_{\theta,t,j})^2}$ would be unbiased for $\ell_t(\theta)$, and the controller would simply run projected OGD \citep{zinkevich2003online,hazan2016oco}:
\begin{equation}
\label{eq:ogd}
\theta_{t+1} = \Pi_\Theta\bigl(\theta_t - \eta\, \nabla \tilde\ell_t(\theta_t)\bigr),
\end{equation}
where $\Pi_\Theta(\cdot)$ is the projection onto the (convex) set of parameters $\Theta$. For an affine forecaster $\smash{\hat h_{\theta,t,j} = \langle \theta, x_{t,j} \rangle}$, the gradient is $\smash{-2 \sum_{j} W_{t,j-1} (e^{\mathrm{obs}}_{t,j} - \langle \theta, x_{t,j} \rangle) x_{t,j}}$. This learner always acts greedily, and Section~\ref{sec:guarantees} shows that its throughput regret is sublinear in the number of rounds. However, this is an idealization as the discrepancy $e_{t,j}$ compares the draft with the target distribution $p_{t,j}$, and the target computes $p_{t,j}$ only at positions it verifies. So, a round that drafts $k_t < K$ tokens gives no feedback at positions $k_t+1,\dots,K$, and observing every position at each round would mean drafting and verifying the full window each time, which defeats the purpose of choosing the horizon. AMTP avoids this: it drafts the full window only in a small random fraction of rounds and re-weights the feedback so that the loss stays unbiased (Section~\ref{sec:controller}).

\subsection{Learning from partial feedback}

A round with horizon $k_t$ reveals feedback only at positions $1,\dots,k_t$. If a greedy controller underestimates late-position survival, its drafts might be short, it may never observe late positions, and it might never corrects the estimate. To address this, AMTP makes two changes to the ideal learner.

\textbf{Exploration.}
With probability $1-\varepsilon$ the controller drafts $k_t = \kappa_t$, and with probability $\varepsilon$ it drafts the full window $k_t = K$ \citep{langford2007epoch}. The randomness here is independent of the draft given $\mathcal{F}_t$, so position $j$ is observed with known probability $\pi_{t,j} = (1-\varepsilon)\mathbf{1}\{j \le \kappa_t\} + \varepsilon \ge \varepsilon$. Other exploration schemes, such as Thompson sampling for contextual bandits \citep{zhang2021feelgoodthompsonsamplingcontextual,anand2025feelgood}, could replace uniform full-horizon exploration. We use $\varepsilon$-greedy exploration because it gives known inclusion probabilities, which is needed for inverse-probability weighting.

\textbf{Inverse-probability weighting.} After verification,  $e^{\mathrm{obs}}_{t,j}$ and $W_{t,j-1}$ are available for every $j \le k_t$. Then, reweighting each observed term by its inclusion probability gives
\begin{equation}
\label{eq:ipw-loss}
\hat\ell_t(\theta) = \sum_{j=1}^{k_t} \frac{W_{t,j-1}}{\pi_{t,j}} \bigl(e^{\mathrm{obs}}_{t,j} - \hat h_{\theta,t,j}\bigr)^2,
\qquad
\mathbb{E}[\hat\ell_t(\theta) | \mathcal{F}_t] = \ell_t(\theta).
\end{equation}
AMTP runs~(\ref{eq:ogd}) with $\hat\ell_t$ in place of $\smash{\tilde\ell_t}$, holding $\smash{W_{t,j-1}}$ and $\smash{\pi_{t,j}}$ fixed in line~8 of Alg~\ref{alg:amtp}. However, relative to the ideal learner, exploration costs some throughput in rounds drafting $K$, and the factor $\smash{1/\pi_{t,j} \le 1/\varepsilon}$ makes gradients noisier: we quantify both in Section~\ref{sec:guarantees}.\looseness=-1

\section{From Prediction Error to Decoding Efficiency}
\label{sec:guarantees}

We now ask how much decoding efficiency AMTP loses by acting on predicted rather than true hazards. First, we show that, for \emph{any} online learner, the throughput regret against the best horizon at each visited history is controlled by the learner's excess prediction loss. We then bound that loss for OGD, first for the ideal full-feedback learner and then for AMTP's partial-feedback learner, and then we bound the price of partial feedback.  Finally, we show how to convert the throughput regret into bounds on expected completion time and speedup. We give complete proofs in Appendix~\ref{app:proofs}.

\subsection{Prediction error to throughput regret}
 
We allow the controller to maximize predicted throughput only approximately, with a selection error $\xi_t \ge 0$ that accommodates exploration:
\begin{equation}
\label{eq:approx-select}
\widehat R_t(k) = \frac{\widehat G_t(k)}{C_t(k)},
\qquad
\widehat R_t(k_t) \ge \max_{1 \le k \le K} \widehat R_t(k) - \xi_t .
\end{equation}
For a comparator sequence $\theta^\star_{1:T}$, define
\[
\mathcal{R}_T = \sum_{t=1}^{T} \bigl(\ell_t(\theta_t) - \ell_t(\theta_t^\star)\bigr),
\qquad
\mathcal{A}_T = \sum_{t=1}^{T} \bigl(\ell_t(\theta_t^\star) - \ell_t^{B}\bigr),
\qquad
\mathcal{D}_T = \sum_{t=1}^{T} \bigl(R_t(k_t^\star) - R_t(k_t)\bigr).
\]
Here $\mathcal{R}_T$ is the learner's regression regret against the comparator, $\mathcal{A}_T \ge 0$ is the comparator's approximation error, and $\mathcal{D}_T$ is the throughput regret against the best horizon at each history the learner visits. This comparator is stronger than the best fixed horizon: it may change from round to round with the context. By Lemma \ref{lem:mtp-calibration}, $\mathcal{R}_T + \mathcal{A}_T = \sum_{t=1}^{T} \mathcal{E}_t(\theta_t) \ge 0$. Write $\Lambda_K = c_{\min}^{-1} \sqrt{\frac{K(K+1)(2K+1)}{6}}$.
 
\begin{theorem}[Dynamic regret transfer]
\label{thm:transfer}
Under the preceding assumptions, for every comparator sequence and every realization of the pre-round histories and selected horizons, we have
\begin{equation}
\label{eq:transfer}
\mathcal{D}_T \le 2\Lambda_K \sqrt{T\,(\mathcal{R}_T + \mathcal{A}_T)} + \sum_{t=1}^{T} \xi_t .
\end{equation}
\end{theorem}

The key idea is to connect calibration to throughput prediction, and then to horizon selection. Telescoping the survival products and applying Lemma~\ref{lem:mtp-calibration} bounds the error in predicted throughput uniformly over horizons. Combining this with Equation \ref{eq:approx-select} bounds the per-round regret:
\[
\max_{1 \le k \le K} \bigl|\widehat R_t(k) - R_t(k)\bigr| \le \Lambda_K \sqrt{\mathcal{E}_t(\theta_t)},
\qquad
R_t(k_t^\star) - R_t(k_t) \le 2\Lambda_K \sqrt{\mathcal{E}_t(\theta_t)} + \xi_t .
\]
Summing over rounds and applying Cauchy--Schwarz gives Equation \ref{eq:transfer}. A regression guarantee therefore becomes a throughput guarantee through the learner's total excess prediction loss, including the comparator's approximation error. Importantly, this does not depend on how $\theta_t$ is produced, which allows us to use any online learning algorithm and transfer their guarantees to our setting.
 
\begin{remark}[Sharpness] The square-root dependence in Eq. \ref{eq:transfer} cannot be improved uniformly. Prop. \ref{prop:mtp-sharpness} with independent proposals and known costs incurs one-round throughput regret $\frac25\sqrt{\mathcal{E}_t(\theta_t)}$.
\end{remark}
 
\begin{remark}[Estimated costs]
Suppose instead that the controller uses $\widehat C_t(k) \ge c_{\min}$, and let $\delta^C_t = \max_{1 \le k \le K} |\widehat C_t(k) - C_t(k)|$. Then, measuring $\xi_t$ relative to $\widehat G_t(k)/\widehat C_t(k)$ gives
\begin{equation}
\label{eq:cost-error}
\mathcal{D}_T \le 2\Lambda_K \sqrt{T\,(\mathcal{R}_T + \mathcal{A}_T)} + \sum_{t=1}^{T} \xi_t + \frac{2(K+1)}{c_{\min}^2} \sum_{t=1}^{T} \delta^C_t .
\end{equation}
\end{remark}

\subsection{Learning the hazards with OGD}

We now bound $\mathcal{R}_T$ for the two learners of Section~\ref{sec:controller}. Both use an affine forecaster $\hat h_{\theta,t,j} = \langle \theta, x_{t,j} \rangle$ on a nonempty closed convex set $\Theta$ of diameter at most $D$. The features are $\mathcal{F}_t$-measurable with $\|x_{t,j}\|_2 \le F$, and forecasts lie in $[0,1]$ for all $\theta \in \Theta$. Initialize $\theta_1 \in \Theta$. The comparator $\theta^\star_t \in \Theta$ is $\mathcal{F}_t$-measurable, and its path length is bounded almost surely by a deterministic quantity, $\sum_{t=1}^{T-1} \|\theta^\star_{t+1} - \theta^\star_t\|_2 \le P_T$, as is standard for dynamic regret \citep{zinkevich2003online,besbes2015nonstationary,zhang2018adaptive}. We assume known costs throughout this subsection.
 
\paragraph{Ideal full feedback.}
The ideal learner of Section~\ref{sec:controller} observes all $K$ positions and always acts greedily, so $\xi_t = 0$. Its gradients $\nabla\tilde\ell_t(\theta_t)$ are unbiased for $\nabla\ell_t(\theta_t)$ and satisfy $\|\nabla\tilde\ell_t(\theta_t)\|_2 \le 2FK$.
 
\begin{corollary}[OGD with \textbf{full} feedback]
\label{cor:ogd-full}
Under these conditions, projected OGD (Equation \ref{eq:ogd}) on the full-feedback loss with constant step size $\eta > 0$ satisfies
\begin{equation}
\label{eq:ogd-full}
\mathbb{E}\,\mathcal{R}_T \le \frac{D^2 + 2DP_T}{2\eta} + 2\eta F^2 K^2 T .
\end{equation}
Under realizability, optimizing $\eta$, for fixed $K, c_{\min}, D, F$, gives $\mathbb{E}\,\mathcal{D}_T = O\bigl(T^{3/4} (1+P_T)^{1/4}\bigr)$.
\end{corollary}
 
Equation~\ref{eq:ogd} is the standard projected-OGD argument with a moving comparator. Theorem~\ref{thm:transfer} and Jensen's inequality then give the final rate. We give the result for any unbiased gradient with bounded second moment in Equation~\ref{eq:ogd}. As Section~\ref{sec:controller} explains, this learner needs the target distribution at all $K$ positions in every round, so this is only our benchmark for an implementable controller.\looseness=-1
 
\textbf{Partial feedback.}
AMTP replaces $\tilde\ell_t$ with the observed loss. Here, the gradients remain unbiased, but each observed term is scaled by $\smash{\tfrac{1}{\pi_{t,j}} \le \tfrac{1}{\varepsilon}}$, so for $g_t = \nabla\hat\ell_t(\theta_t)$ we have $\mathbb{E}[\|g_t\|_2^2 | \mathcal{F}_t] \le \frac{4F^2K^2}{\varepsilon}$, where exploration rounds contribute $\mathbb{E}\sum_{t=1}^{T} \xi_t \le \varepsilon T \frac{K+1}{c_{\min}}$ to the selection error.\looseness=-1

\begin{corollary}[OGD with \textbf{partial} feedback]
\label{cor:ogd-partial}
Under these conditions, Algorithm~\ref{alg:amtp} with known costs, exploration rate $\varepsilon \in (0,1]$ and constant step size $\eta > 0$ satisfies
\begin{equation}
\label{eq:ogd-partial}
\mathbb{E}\,\mathcal{R}_T \le \frac{D^2 + 2DP_T}{2\eta} + \frac{2\eta F^2 K^2 T}{\varepsilon} .
\end{equation}
Under realizability, optimizing $\eta$, for fixed $K, c_{\min}, D, F$, gives $\mathbb{E}\,\mathcal{D}_T
= O\bigl(T^{4/5} (1+P_T)^{1/5}\bigr)$.
\end{corollary}
 
The projected-OGD bound with the second-moment estimate above gives~\eqref{eq:ogd-partial}, and Theorem~\ref{thm:transfer} with Jensen's inequality and a balancing argument gives the final rate. This tuning uses the number of rounds $T$ and a known bound $P_T$ on comparator variation.
 
\textbf{The price of partial feedback.}
Equations \ref{eq:ogd-full} and \ref{eq:ogd-partial} differ only by the factor $\frac1\varepsilon$ in the gradient-noise term, and the rates differ only via the selection-error term $\varepsilon T$. Here, exploration must be frequent enough to keep gradient noise in check, but it also drafts the full window when a shorter one may be better. Balancing these effects moves the rate from $T^{3/4}(1+P_T)^{1/4}$ to $T^{4/5}(1+P_T)^{1/5}$, which are  both sublinear when $P_T = o(T)$. So, the average throughput approaches that of the best horizon at each history, and AMTP pays a vanishing gap for learning from the feedback of its horizon choices.\looseness=-1

\subsection{From throughput regret to speedup}

We measure the throughput regret per round, as the goal is to minimize the time to produce an output. We fix an integer output budget $N \ge 1$, and compare AMTP with a baseline under matched prompts, the same target output law and the same timing scope. Let $\smash{J_N^{\mathrm{base}}}$ denote the baseline's completion time and define its effective rate $\smash{r_{\mathrm{base}}(N) = N / \mathbb{E} J_N^{\mathrm{base}}}$, assuming $\smash{0 < \mathbb{E} J_N^{\mathrm{base}} < \infty}$.
 
\begin{corollary}[Completion time and speedup]
\label{cor:speedup}
Given $\max_{1 \le k \le K} R_t(k) \ge \rho > 0$ at each visited history, for some constant $\rho$, let the realized round cost $c_t \ge 0$ satisfy $\mathbb{E}[c_t| \mathcal{F}_t, k_t] = C_t(k_t)$, and let\looseness=-1
\begin{equation}
\label{eq:completion}
\sigma_N = \inf\Bigl\{m \ge 1 : \sum_{t=1}^{m} L_t(k_t) \ge N\Bigr\},
\qquad
J_N = \sum_{t=1}^{\sigma_N} c_t .
\end{equation}
Extend the same policy to $N$ rounds and suppose that $\mathbb{E}\,\mathcal{D}_N \le B_N$. Then, the expected completion time and the speedup, defined as the ratio of expected completion times, satisfy
\begin{equation}
\label{eq:speedup}
\mathbb{E} J_N \le \frac{N + K + c_{\max} B_N}{\rho},
\qquad
S_N := \frac{\mathbb{E} J_N^{\mathrm{base}}}{\mathbb{E} J_N} \ge \frac{\rho / r_{\mathrm{base}}(N)}{1 + K/N + c_{\max} B_N / N} .
\end{equation}
In particular, $S_N > 1$ whenever $\rho / r_{\mathrm{base}}(N) > 1 + K/N + c_{\max} B_N/N$.
\end{corollary}
 
To see why, let $\Delta_t = R_t(k_t^\star) - R_t(k_t) \ge 0$. The assumption on $\rho$ gives $G_t(k_t) \ge \rho\, C_t(k_t) - c_{\max} \Delta_t$. Each round commits at least one token, so $\sigma_N \le N$, and the final round overshoots the budget by at most $K$ tokens. Since $\{\sigma_N \ge t\}$ is determined before round $t$, conditional expectations gives
\[
\rho\, \mathbb{E} J_N \le \mathbb{E}\sum_{t=1}^{\sigma_N} L_t(k_t) + c_{\max}\, \mathbb{E}\sum_{t=1}^{\sigma_N} \Delta_t \le N + K + c_{\max} B_N ,
\]
which is the bound in (\ref{eq:speedup}). Dividing the baseline's completion time by this gives the claim. 

 \paragraph{Speedup guarantee of AMTP.} Corollary~\ref{cor:ogd-partial} with $T = N$ gives $\smash{B_N = O(N^{4/5}(1+P_N)^{1/5})}$, so the penalty in Equation \ref{eq:speedup} is $O(K/N + ({(1+P_N)}/{N})^{1/5})$. If $P_N = o(N)$ and $r_{\mathrm{base}}(N) \to r_0 > 0$, the penalty vanishes and $\liminf_{N \to \infty} S_N \ge {\rho}/{r_0}$. Acceptance and round cost thus set the achievable speedup ${\rho}/{r_0}$, and learning the horizon costs only a vanishing fraction of it.\looseness=-1

\section{Experiments}
\label{sec:experiments}

We evaluate AMTP on three self-speculative multi-token-prediction families (self-distilled Qwen3-4B, Meta's facebook/multi-token-prediction 7B, and Medusa heads on Llama-3.1-8B-Instruct) and two cross-model pairings with a separate, larger target (self-distilled 4B $\to$ Qwen3-14B; Meta MTP 7B $\to$ Llama-2-13B), on AIME 2025 and GSM8K (full setup in Table~\ref{tab:mtp-experiment-matrix}, Appendix~\ref{app:setup}). All runs disable KV caching, so verification recomputes the full context each round; each condition is a single unseeded run, and we compare conditions only within a model since hardware differs across runs. The fixed-window baseline uses the largest window each draft supports ($k{=}16$, $4$, $6$); the adaptive controller is the affine hazard predictor and cost model of Section~\ref{sec:controller} with $\varepsilon$-bounded exploration, updated by one gradient step per round (hyperparameters in Table~\ref{tab:cb-config}, Appendix~\ref{app:setup}).

\begin{table*}[hbt!]
\centering
\footnotesize
\setlength{\tabcolsep}{6pt}
\begin{tabular}{@{}llrrrr@{}}
\toprule
& & \multicolumn{2}{c}{AIME 2025} & \multicolumn{2}{c}{GSM8K} \\
\cmidrule(lr){3-4}\cmidrule(lr){5-6}
Model & Decoding & tok/s & Speed-up & tok/s & Speed-up \\
\midrule
Self-distilled Qwen3-4B & Plain (next-token) & 18.8 & 1.00$\times$ & 16.2 & 1.00$\times$ \\
 & Fixed $k{=}16$ & 34.0 & 1.81$\times$ & 38.6 & 2.39$\times$ \\
 & Adaptive (ours) & \textbf{34.5} & \textbf{1.83$\times$} & 38.3 & 2.37$\times$ \\
\midrule
Meta MTP 7B & Plain (next-token) & 12.5 & 1.00$\times$ & 4.7\textsuperscript{b} & 1.00$\times$ \\
 & Fixed $k{=}4$ & 37.7 & 3.01$\times$ & 19.2\textsuperscript{b} & 4.09$\times$ \\
 & Adaptive (ours) & 31.9 & 2.55$\times$ & 19.4\textsuperscript{b} & 4.13$\times$ \\
\midrule
Meta MTP 7B $\to$ Llama-2-13B & Plain (next-token) & 6.2 & 1.00$\times$ & 2.7\textsuperscript{b} & 1.00$\times$ \\
 & Fixed $k{=}4$ & 7.6 & 1.23$\times$ & 4.7\textsuperscript{b} & 1.75$\times$ \\
 & Adaptive (ours) & \textbf{28.2} & \textbf{4.55$\times$} & \textbf{11.6}\textsuperscript{b} & \textbf{4.30$\times$} \\
\midrule
Medusa, Llama-3.1-8B & Plain (next-token) & 10.7 & 1.00$\times$ & 13.0 & 1.00$\times$ \\
 & Fixed $k{=}6$ & 23.6 & 2.20$\times$ & 20.0 & 1.54$\times$ \\
 & Adaptive (ours) & \textbf{25.3} & \textbf{2.37$\times$} & 19.4 & 1.50$\times$ \\
\midrule
Self-dist.\ 4B $\to$ Qwen3-14B & Plain (next-token) & 10.4 & 1.00$\times$ & 6.4\textsuperscript{a} & 1.00$\times$ \\
 & Fixed $k{=}16$ & 22.5 & 2.17$\times$ & 21.3 & 3.33$\times$ \\
 & Adaptive (ours) & \textbf{24.7} & \textbf{2.38$\times$} & 21.5 & 3.36$\times$ \\
\bottomrule
\end{tabular}
\caption{Decoding throughput (tokens/s) and speed-up over plain decoding for each draft/target pairing on AIME 2025 and GSM8K, with AMTP in bold if it exceeds the best fixed window by more than 1\%. \textsuperscript{a}Plain decoding run reached 1{,}166 of 1{,}319 GSM8K problems before its timeout: speed-up computed over that subset.  \textsuperscript{b}Meta MTP's base checkpoints do not emit an end-of-sequence token after answering, so GSM8K runs for this pairing are capped at 100 problems, 100 rounds each.}
\label{tab:throughput}
\end{table*}

\textbf{Discussion.} AMTP helps most when drafting and verification differ in cost and when reasoning traces are long and variable in length. The cross-model pairings, where draft and target are separate models, show the largest gains over the fixed window: 
Meta MTP $7\mathrm{B}\to$Llama-2-13B improves throughput by $268\%$ on AIME and $147\%$ on GSM8K (over its best fixed window), and self-distilled $4\mathrm{B}\to$Qwen3-14B by $10\%$ on AIME while staying within $1\%$ of the best fixed window on GSM8K.  In the self-speculative pairings, draft and target share one forward pass, so a longer horizon costs little extra and there is less to gain: self-distilled 4B and Medusa improve by $2$--$7\%$ on AIME and stay within $3\%$ of the fixed window on GSM8K. The one regression is Meta MTP 7B on AIME ($-15\%$). Meta's base checkpoint does not reliably stop after answering, and every round of a longer generation costs more. We note, however, that its acceptance rate was higher under adaptive control throughout. Across all pairings, AMTP drafts fewer tokens per round while committing a similar number, allowing  it to accept a higher fraction of what it drafts than any fixed window, converting wasted draft computation into completed tokens even where the net throughput effect is small.\looseness=-1

\begin{table}[hbt!]
\centering
\footnotesize
\setlength{\tabcolsep}{3pt}
\begin{tabular}{@{}llrrrrrrr@{}}
\toprule
& & \multicolumn{2}{c}{Accept.\ (\%)} & \multicolumn{2}{c}{Tokens/round} & \multicolumn{3}{c}{Accuracy (\%)} \\
\cmidrule(lr){3-4}\cmidrule(lr){5-6}\cmidrule(lr){7-9}
Model & Eval & Fixed & Adapt. & Fixed & Adapt. & Plain & Fixed & Adapt. \\
\midrule
Self-distilled 4B & GSM8K & 23 & 32 & 4.67 & 4.63 & 89.4 & 85.0 & 85.7 \\
Medusa 8B & GSM8K & 33 & 48 & 2.99 & 2.90 & 83.2 & 79.6 & 79.2 \\
Self-dist.\ 4B $\to$ 14B & GSM8K & 21 & 29 & 4.33 & 4.33 & 92.8\textsuperscript{a} & 92.9 & 93.5 \\
Self-distilled 4B & AIME & 18 & 34 & 3.84 & 3.83 & 20.0 & 30.0 & 26.7 \\
Meta MTP 7B & AIME & 50 & 57 & 3.00 & 3.03 & 0.0 & 0.0 & 0.0 \\
Meta MTP 13B & AIME & 44 & 52 & 2.76 & 2.51 & 0.0 & 0.0 & 0.0 \\
Medusa 8B & AIME & 22 & 46 & 2.34 & 2.28 & 0.0 & 0.0 & 0.0 \\
Self-dist.\ 4B $\to$ 14B & AIME & 14 & 28 & 3.23 & 3.20 & 20.0 & 20.0 & 23.3 \\
\bottomrule
\end{tabular}
\caption{Draft acceptance rate, tokens committed per verification round, and accuracy for plain decoding, the best fixed window and AMTP on AIME 2025 and GSM8K. \textsuperscript{a}Plain decoding run reached 1{,}166 of 1{,}319 GSM8K problems before its time limit; accuracy computed over that subset.}
\label{tab:quality}
\end{table}

\begin{table}[t]
\centering
\small
\setlength{\tabcolsep}{5pt}
\begin{tabular}{@{}llrrrr@{}}
\toprule
Model & Eval & Mean $k_t$ & Std.\ & Range & Distinct $k$ \\
\midrule
Self-distilled 4B & GSM8K & 11.21 & 2.73 & 1--16 & 16 \\
Meta MTP 7B & AIME & 3.54 & 0.69 & 1--4 & 4 \\
Meta MTP 13B & AIME & 2.93 & 0.99 & 1--4 & 4 \\
Medusa 8B & GSM8K & 3.95 & 1.02 & 1--6 & 6 \\
\bottomrule
\end{tabular}
\caption{Distribution of the draft window $k_t$ chosen by AMTP over all verification rounds of a run. We remark here that the window is not constant: it varies across rounds within a single generation and across generations.}
\label{tab:ktstats}
\end{table}

\pgfplotsset{
  paper/.style={tick label style={font=\footnotesize}, label style={font=\footnotesize}, title style={font=\small},
    legend style={font=\scriptsize, draw=none, fill=none}, ymajorgrids, grid style={gray!25}, axis line style={black!70}},
}

\begin{figure*}[hbt!]
\centering
\begin{tikzpicture}
\begin{groupplot}[group style={group size=4 by 1, horizontal sep=0.9cm}, paper,
  width=0.3\textwidth, height=3.8cm,
  label style={font=\footnotesize},
  tick label style={font=\scriptsize}]
\nextgroupplot[title={{\footnotesize\shortstack{Self-distilled Qwen3-4B\\GSM8K}}}, xlabel={Verification round}, ylabel={Mean $k_t$}, xmin=0, ymin=3, ymax=13]
\addplot[cAdapt, thick, mark=*, mark size=1pt, error bars/.cd, y dir=both, y explicit, error bar style={cAdapt!45, line width=0.4pt}] coordinates {(5,11.411) +- (0,2.609) (15,11.496) +- (0,2.634) (25,11.481) +- (0,2.647) (35,11.400) +- (0,2.651) (45,11.295) +- (0,2.746) (55,11.140) +- (0,2.788) (65,11.015) +- (0,2.790) (75,10.733) +- (0,2.854) (85,10.540) +- (0,2.816) (95,10.462) +- (0,2.991) (105,10.372) +- (0,2.952) (115,10.379) +- (0,2.927) (125,10.287) +- (0,2.714) (135,10.147) +- (0,2.980) (145,10.337) +- (0,2.841) (155,10.199) +- (0,2.878) (165,10.119) +- (0,2.546) (175,9.985) +- (0,2.885) (185,9.639) +- (0,2.836) (195,10.323) +- (0,3.066) (205,9.779) +- (0,2.970) (215,10.214) +- (0,3.051) (225,9.502) +- (0,2.942) (235,9.687) +- (0,2.735) (245,9.684) +- (0,2.896) (255,9.614) +- (0,2.550) (265,9.936) +- (0,2.871) (275,9.624) +- (0,2.349) (285,9.638) +- (0,2.722) (295,9.278) +- (0,2.541)};
\nextgroupplot[title={{\footnotesize\shortstack{Medusa Llama-3.1-8B\\GSM8K}}}, xlabel={Verification round}, xmin=0, ymin=1, ymax=5]
\addplot[cAdapt, thick, mark=*, mark size=1pt, error bars/.cd, y dir=both, y explicit, error bar style={cAdapt!45, line width=0.4pt}] coordinates {(5,4.070) +- (0,0.909) (15,4.139) +- (0,0.938) (25,4.117) +- (0,0.903) (35,4.144) +- (0,0.895) (45,4.129) +- (0,0.946) (55,4.070) +- (0,0.905) (65,4.079) +- (0,0.950) (75,3.895) +- (0,0.900) (85,3.880) +- (0,0.988) (95,3.828) +- (0,1.021) (105,3.595) +- (0,1.062) (115,3.770) +- (0,1.132) (125,3.526) +- (0,0.963) (135,3.353) +- (0,0.953) (145,3.531) +- (0,1.175) (155,3.444) +- (0,1.086) (165,3.465) +- (0,1.095) (175,3.427) +- (0,1.236) (185,3.169) +- (0,1.035) (195,3.192) +- (0,1.128) (205,3.077) +- (0,1.192) (215,3.072) +- (0,1.199) (225,2.983) +- (0,1.303) (235,3.050) +- (0,1.327) (245,2.900) +- (0,1.155) (255,2.933) +- (0,1.186) (265,3.000) +- (0,1.237) (275,2.975) +- (0,1.312) (285,2.933) +- (0,1.186) (295,2.792) +- (0,1.162)};
\nextgroupplot[title={{\footnotesize\shortstack{Meta MTP 7B\\AIME 2025}}}, xlabel={Verification round}, xmin=0, ymin=1, ymax=5]
\addplot[cAdapt, thick, mark=*, mark size=1pt, error bars/.cd, y dir=both, y explicit, error bar style={cAdapt!45, line width=0.4pt}] coordinates {(5,3.000) +- (0,0.941) (15,3.108) +- (0,0.868) (25,3.012) +- (0,0.885) (35,3.171) +- (0,0.863) (45,3.237) +- (0,0.832) (55,3.337) +- (0,0.857) (65,3.254) +- (0,0.832) (75,3.333) +- (0,0.774) (85,3.287) +- (0,0.772) (95,3.257) +- (0,0.782) (105,3.370) +- (0,0.729) (115,3.338) +- (0,0.745) (125,3.486) +- (0,0.699) (135,3.536) +- (0,0.651) (145,3.477) +- (0,0.699) (155,3.479) +- (0,0.724) (165,3.507) +- (0,0.664) (175,3.510) +- (0,0.702) (185,3.536) +- (0,0.661) (195,3.516) +- (0,0.623) (205,3.532) +- (0,0.631) (215,3.600) +- (0,0.598) (225,3.592) +- (0,0.602) (235,3.660) +- (0,0.551) (245,3.667) +- (0,0.620) (255,3.608) +- (0,0.544) (265,3.655) +- (0,0.548) (275,3.685) +- (0,0.557) (285,3.700) +- (0,0.493) (295,3.708) +- (0,0.474) (305,3.650) +- (0,0.560) (315,3.700) +- (0,0.616) (325,3.658) +- (0,0.542) (335,3.591) +- (0,0.625) (345,3.645) +- (0,0.568) (355,3.736) +- (0,0.501) (365,3.673) +- (0,0.679) (375,3.845) +- (0,0.363) (385,3.755) +- (0,0.528) (395,3.727) +- (0,0.542)};
\nextgroupplot[title={{\footnotesize\shortstack{Meta MTP 13B\\AIME 2025}}}, xlabel={Verification round}, xmin=0, ymin=1, ymax=5]
\addplot[cAdapt, thick, mark=*, mark size=1pt, error bars/.cd, y dir=both, y explicit, error bar style={cAdapt!45, line width=0.4pt}] coordinates {(5,2.152) +- (0,0.973) (15,2.571) +- (0,1.085) (25,2.295) +- (0,0.994) (35,2.677) +- (0,0.974) (45,2.727) +- (0,1.057) (55,2.833) +- (0,0.999) (65,2.804) +- (0,0.966) (75,2.909) +- (0,0.904) (85,2.863) +- (0,0.944) (95,2.800) +- (0,1.093) (105,2.911) +- (0,0.990) (115,2.962) +- (0,0.898) (125,3.143) +- (0,0.889) (135,3.000) +- (0,0.901) (145,3.071) +- (0,0.953) (155,3.214) +- (0,0.866) (165,3.200) +- (0,0.861) (175,3.314) +- (0,0.692) (185,3.243) +- (0,0.842) (195,3.254) +- (0,0.803) (205,3.220) +- (0,0.840) (215,3.060) +- (0,0.767) (225,3.098) +- (0,0.860) (235,3.200) +- (0,0.883) (245,3.000) +- (0,0.784) (255,3.425) +- (0,0.844) (265,3.350) +- (0,0.700) (275,3.375) +- (0,0.897) (285,3.150) +- (0,0.770) (295,3.312) +- (0,0.738) (305,3.367) +- (0,0.556) (315,3.500) +- (0,0.682) (325,3.567) +- (0,0.679) (335,3.567) +- (0,0.626) (345,3.400) +- (0,0.675) (355,3.500) +- (0,0.682) (365,3.464) +- (0,0.793) (375,2.900) +- (0,1.210) (385,2.900) +- (0,1.071) (395,2.950) +- (0,0.945) (405,3.300) +- (0,0.979) (415,3.650) +- (0,0.933) (425,2.667) +- (0,1.283) (435,3.100) +- (0,1.287) (445,2.700) +- (0,1.418) (455,3.000) +- (0,1.333) (465,2.900) +- (0,1.449) (475,3.500) +- (0,0.837)};
\end{groupplot}
\end{tikzpicture}
\caption{Draft window chosen by AMTP, averaged over generations in bins of 10 verification rounds (points; bars show one standard deviation across rounds in the bin). The window adapts as generation proceeds instead of staying at a single value.}
\label{fig:ktmean}
\end{figure*}
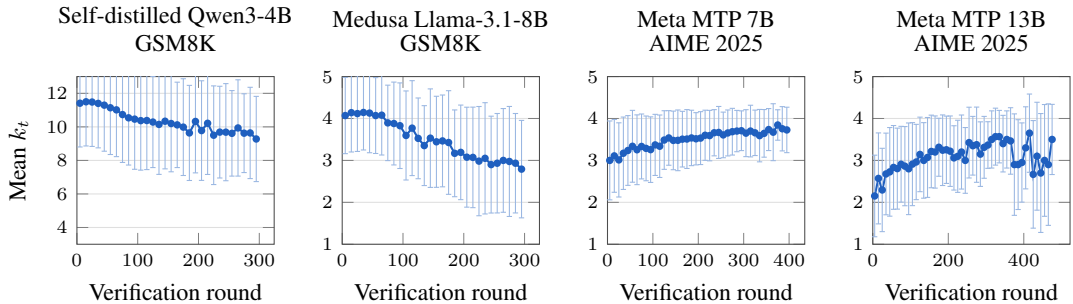

\section{Conclusion and Future Work}
\label{sec:conclusion}

In this work, we study how far MTP should speculate. Since all MTP heads fire in a single forward pass, the horizon must be chosen before drafting, and the right choice changes with context. We showed that acceptance survival determines the expected number of committed tokens at every horizon for any causal proposal, and built AMTP, an online controller that learns survival from verification feedback and picks the horizon with the highest predicted throughput, without retraining any model or changing the exact verifier. We show that any online learner's prediction error translates into throughput regret against the best horizon for each context, and in turn into a bound on wall-clock completion time. Empirically, AMTP decodes up to 4.1$\times$ faster than standard decoding and matches or outperforms the fixed window in 8 of 10 settings. 

Several limitations point to future work. Our experiments decode one sequence at a time without a KV-cache: batched serving with caching changes the cost of each horizon and is a natural direction. Comparing hazard-based and direct survival prediction, and extending horizon control to tree-structured drafts, are promising directions. Finally, it would be interesting to extend our work to adaptive diffusion-based speculative methods.

\section*{AI Use Statement}
Generative AI tools assisted with editing the manuscript, generating code for the experiments, identifying and summarizing related literature, and proof-checking the mathematical formulations and proofs. The authors are responsible for checking the final claims, proofs, and citations.

\section*{Acknowledgements}
We thank Adam Wierman, Jacob Abernethy, Matthew Fahrbach, and Sarah Liaw for helpful comments. EA's research is supported by NSF Grant CCF 2338816.

\bibliography{main}
\bibliographystyle{plainnat}

\newpage

\appendix

\textbf{Outline of the Appendices}.
\begin{itemize}
    \item Section \ref{app:motivation-proofs} provides some motivation for adaptive horizon selection;
    \item Section \ref{app:proofs} presents the proof of our main results;
    \item Section \ref{app:costs} shows how to estimate the costs and throughput measurements;
    \item Section \ref{app:setup} explains our experimental setup and provides further results.
\end{itemize}

\section{Why Longer Drafts Degrade}
\label{app:motivation-proofs}

Draft qualities can deteriorate with the horizon even if every head is accurate on its own. For independent heads, the block mismatch has two sources: marginal fitting error, and dependence among future target tokens that the heads cannot represent, neither of which individually determines the throughput-optimal horizon, which depends on the verifier's acceptance profile and the round cost.\looseness=-1
 
Fix a round, condition on its committed context, and suppress the round index. Let $P^{(k)}$ and $Q^{(k)}$ be the target and proposal laws over the next $k$ tokens, assumed to be prefix marginals of common laws $P$ and $Q$, and let $P_j$ be the position-$j$ target marginal. The vocabulary is finite and logarithms are natural; KL identities use the usual extended-value convention.
  
\begin{proposition}[Dependence and marginal fit]
\label{prop:structure}
For independent heads $Q^{(k)}=\prod_{j=1}^k q_j$, let $\mathrm{TC}_P(Y_{1:k}) := \KL\bigl(P^{(k)}\|\prod_{j=1}^k P_j\bigr)$ be the total correlation of the target block. Then, 
\begin{equation}
\KL(P^{(k)}\|Q^{(k)}) = \mathrm{TC}_P(Y_{1:k}) + \sum_{j=1}^k \KL(P_j\|q_j),
\label{eq:independent-kl-appendix}
\end{equation}
and, at each position $j$,
\begin{equation}
\E_{Y_{<j}\sim P^{(j-1)}}\bigl[\KL\bigl(p_j(\cdot\mid Y_{<j})\|q_j\bigr)\bigr] = I_P(Y_j;Y_{<j}) + \KL(P_j\|q_j).
\label{eq:motivation-mi}
\end{equation}
\end{proposition}
 
\begin{proof}
Insert $\prod_j P_j(y_j)$ into the likelihood ratio:
\[
\log\frac{P^{(k)}(y_{1:k})}{\prod_j q_j(y_j)} = \log\frac{P^{(k)}(y_{1:k})}{\prod_j P_j(y_j)} + \sum_j \log\frac{P_j(y_j)}{q_j(y_j)},
\]
and take expectations under $P^{(k)}$.  Inserting it in the conditional ratio $\frac{p_j(Y_j| Y_{<j})}{q_j(Y_j)}$ gives Eq.~\ref{eq:motivation-mi}.
\end{proof}
 
Even exact marginal heads, $q_j = P_j$, therefore maintain the total correlation as their block KL divergence. The term $I_P(Y_j;Y_{<j})$ is the information about earlier target tokens that an independent head cannot use, and heads that condition on earlier proposed tokens can reduce it.
 
We now explain how the block mismatch grows with the horizon. Let $\delta(k) := \mathrm{TV}(P^{(k)},Q^{(k)})$. Because marginalization contracts total variation, we have
\begin{align*}
2\delta(k) &= \sum_{y_{1:k}}\Bigl|\sum_v \bigl(P^{(k+1)}(y_{1:k},v) - Q^{(k+1)}(y_{1:k},v)\bigr)\Bigr|\\
&\le \sum_{y_{1:k},v}\bigl|P^{(k+1)}(y_{1:k},v) - Q^{(k+1)}(y_{1:k},v)\bigr| = 2\delta(k+1).
\label{eq:motivation-tv-monotone}
\end{align*} Similarly, block mismatch  caps survival: with the conditional kernels defined along a path, the probability that all $k$ sequential checks accept is
\begin{equation}
s_k = \sum_{y_{1:k}}\prod_{j=1}^k \min\{p_j(y_j\mid y_{<j}),\, q_j(y_j\mid y_{<j})\} \le 1 - \delta(k),
\label{eq:block-overlap}
\end{equation}
since each product of minima is at most the minimum of the two products. 
 
\begin{example}[Perfect marginals, bounded gain]
\label{ex:repeated-bit}
Let the target draw one fair bit $B$ and repeat it, $Y_j = B$ for all $j$, so $P^{(k)}(0^k) = P^{(k)}(1^k) = 1/2$. Independent heads with exact marginals draw a fresh fair bit at each position, so $Q^{(k)}(y_{1:k}) = 2^{-k}$. Then
\[
\delta(k) = 1 - 2^{1-k},
\qquad
\KL(P^{(k)}\|Q^{(k)}) = (k-1)\log 2 .
\]
The first proposal is always accepted. Given that the prefixes have been accepted, each later target conditional is a point mass at the first bit while the proposal is uniform, so each later position is accepted with probability $1/2$. Therefore, we have that
$\smash{s_j = 2^{1-j}}$, and
$\smash{G(k) = 1 + \sum_{j=1}^k 2^{1-j} = 3 - 2^{1-k} < 3}$, showing that exact marginals can coexist with block mismatch approaching $1$ and an expected gain that never reaches $3$ tokens, no matter how long the draft.\looseness=-1
\end{example}

\section{Proofs of the Efficiency Guarantees}
\label{app:proofs}
This section proves the results stated in the main text, in the order they appear. Throughout, write $E_t = \ell_t(\theta_t) - \ell_t^{\mathrm B}$ and $\widehat h_{t,r} = \widehat h_{\theta_t,t,r}$. Then, by Lemma~\ref{lem:mtp-calibration}, we have $\sum_{t=1}^T E_t = \mathcal R_T + \mathcal A_T$.
 
\textbf{Predictor class.} We use an affine forecaster $\widehat h_{\theta,t,j} = \langle\theta, x_{t,j}\rangle$. For instance, consider simplex weights $\theta$ and features $x_{t,j}\in[0,1]^d$: here, the loss is convex in $\theta$ and the forecasts stay in $[0,1]$. 
 
\begin{proof}[Proof of Lemma~\ref{lem:mtp-calibration}]
Write $e_{t,j}(Y_{<j}) - \widehat h_{\theta,t,j} = \bigl(e_{t,j}(Y_{<j}) - h_{t,j}\bigr) + \bigl(h_{t,j} - \widehat h_{\theta,t,j}\bigr)$ and expand each square. By the definition of the hazard~\eqref{eq:mtp-hazard}, the cross term vanishes:
\[
\mathbb E_{Q_t}\bigl[W_{t,j-1}\bigl(e_{t,j}(Y_{<j}) - h_{t,j}\bigr)\bigr] = 0 .
\]
The first square summed over $j$ is $\ell_t^{\mathrm B}$, and the second is multiplied by $\mathbb E_{Q_t}[W_{t,j-1}] = s_{t,j-1}$.
\end{proof}
 
\begin{proof}[Proof of Theorem~\ref{thm:transfer}]
First, telescoping the survival products gives
\[
\widehat s_{t,j} - s_{t,j} = \sum_{r=1}^{j} s_{t,r-1}\,\bigl(h_{t,r} - \widehat h_{t,r}\bigr) \prod_{u=r+1}^{j}\bigl(1 - \widehat h_{t,u}\bigr).
\]
Every factor in the product lies in $[0,1]$, so summing over $j \le k$, applying Cauchy--Schwarz, and using $s_{t,r-1} \le 1$ and Lemma~\ref{lem:mtp-calibration},
\begin{align}
\bigl|\widehat G_t(k) - G_t(k)\bigr|
&\le \sum_{r=1}^{k} (k-r+1)\, s_{t,r-1}\, \bigl|\widehat h_{t,r} - h_{t,r}\bigr| \nonumber\\
&\le \Bigl(\sum_{r=1}^{k} (k-r+1)^2 s_{t,r-1}\Bigr)^{1/2} \Bigl(\sum_{r=1}^{k} s_{t,r-1}\bigl(\widehat h_{t,r} - h_{t,r}\bigr)^2\Bigr)^{1/2} \nonumber
\\
&\le \sqrt{\frac{k(k+1)(2k+1)}{6}}\,\sqrt{E_t}.
\label{eq:mtp-gain-error}
\end{align}
 
Dividing by $C_t(k) \ge c_{\min}$ and using $k \le K$ gives $\max_k |\widehat R_t(k) - R_t(k)| \le \Lambda_K \sqrt{E_t}$. Combining this with the selection rule in~\eqref{eq:approx-select} gives
\[
R_t(k_t^\star) - R_t(k_t)
\le \bigl[R_t(k_t^\star) - \widehat R_t(k_t^\star)\bigr] + \bigl[\widehat R_t(k_t^\star) - \widehat R_t(k_t)\bigr] + \bigl[\widehat R_t(k_t) - R_t(k_t)\bigr]
\le 2\Lambda_K\sqrt{E_t} + \xi_t .
\]
 Then, by Cauchy--Schwarz, we have $\sum_{t=1}^T \sqrt{E_t} \le \sqrt{T \sum_t E_t} = \sqrt{T(\mathcal R_T + \mathcal A_T)}$.
\end{proof}
 
\textbf{A general projected-OGD bound.} Both learning corollaries follow from the standard projected-OGD argument \citep{zinkevich2003online} with a moving comparator and stochastic gradients.
 
\begin{lemma}[Projected OGD with a moving comparator]
\label{lem:mtp-ogd}
\label{cor:mtp-ogd}
Let $\Theta$ be a nonempty closed convex set of diameter at most $D$, let $\theta_1 \in \Theta$, and let each $\ell_t$ be convex and differentiable on $\Theta$. Suppose $\theta_{t+1} = \Pi_\Theta(\theta_t - \eta g_t)$ with $\mathbb E[g_t | \mathcal F_t] = \nabla \ell_t(\theta_t)$ and $\mathbb E[\|g_t\|_2^2 | \mathcal F_t] \le M^2$. Let $\theta_t^\star \in \Theta$ be $\mathcal F_t$-measurable with $\sum_{t<T} \|\theta_{t+1}^\star - \theta_t^\star\|_2 \le P_T$ almost surely, for a deterministic $P_T$. Then
\begin{equation}
\mathbb E\,\mathcal R_T \le \frac{D^2 + 2DP_T}{2\eta} + \frac{\eta M^2 T}{2}.
\label{eq:mtp-ogd-regret}
\end{equation}
With $\eta = \sqrt{\frac{D^2 + 2DP_T}{M^2 T}}$ in the nondegenerate case, Theorem~\ref{thm:transfer} and Jensen's inequality give
\begin{equation}
\mathbb E\,\mathcal D_T \le 2\Lambda_K \sqrt{T\Bigl(M\sqrt{T(D^2 + 2DP_T)} + \mathbb E\,\mathcal A_T\Bigr)} + \mathbb E\sum_{t=1}^T \xi_t .
\label{eq:mtp-ogd-throughput}
\end{equation}
\end{lemma}
 
\begin{proof}
Nonexpansiveness of the projection gives
\[
\langle g_t, \theta_t - \theta_t^\star\rangle
\le \frac{\|\theta_t - \theta_t^\star\|_2^2 - \|\theta_{t+1} - \theta_t^\star\|_2^2}{2\eta} + \frac{\eta}{2}\|g_t\|_2^2 .
\]
By convexity and conditional unbiasedness, the expectation of the left-hand side is at least $\mathbb E[\ell_t(\theta_t) - \ell_t(\theta_t^\star)]$. To telescope against a moving comparator, we use
\[
\Bigl|\|\theta_{t+1} - \theta_{t+1}^\star\|_2^2 - \|\theta_{t+1} - \theta_t^\star\|_2^2\Bigr| \le 2D\,\|\theta_{t+1}^\star - \theta_t^\star\|_2 .
\]
Summing over $t$ proves~\eqref{eq:mtp-ogd-regret}. Substituting the chosen $\eta$ gives $\mathbb E\,\mathcal R_T \leq M\sqrt{T(D^2 + 2DP_T)}$, and~\eqref{eq:mtp-ogd-throughput} follows from Theorem~\ref{thm:transfer} and Jensen's inequality.
\end{proof}
 
\begin{proof}[Proof of Corollary~\ref{cor:ogd-full} (full feedback)]Here, $g_t = -2\sum_{j=1}^K W_{t,j-1}(e^{\mathrm{obs}}_{t,j} - \langle\theta_t, x_{t,j}\rangle)\,x_{t,j}$ is unbiased for $\nabla\ell_t(\theta_t)$, and since $W_{t,j-1} \le 1$, $|e^{\mathrm{obs}}_{t,j} - \widehat h_{t,j}| \le 1$ and $\|x_{t,j}\|_2 \le F$, we have $\|g_t\|_2 \le 2FK$. Apply Lemma~\ref{lem:mtp-ogd} with $M = 2FK$ and $\xi_t = 0$. Under realizability, $\mathbb E\,\mathcal D_T = O\bigl(T^{3/4}(1+P_T)^{1/4}\bigr)$ for fixed $K, c_{\min}, D, F$.
\end{proof}
 
\begin{proof}[Proof of Corollary~\ref{cor:ogd-partial} (partial feedback)]
\emph{Unbiased loss and gradient.} Fix $\theta$. Because the exploration coin is independent of the draft given $\mathcal F_t$, for each position
\[
\mathbb E\Bigl[\frac{\mathbf 1\{k_t \ge j\}}{\pi_{t,j}}\, W_{t,j-1}\bigl(e^{\mathrm{obs}}_{t,j} - \widehat h_{\theta,t,j}\bigr)^2 \Bigm| \mathcal F_t\Bigr]
= \mathbb E_{Q_t}\Bigl[W_{t,j-1}\bigl(e_{t,j}(Y_{<j}) - \widehat h_{\theta,t,j}\bigr)^2\Bigr].
\]
This discrepancy is needed only when the indicator is one. Summing over $j$ shows that the observed loss is conditionally unbiased, and differentiating shows this for the gradient of the affine predictor.\looseness=-1
 
\emph{Second moment.} Using $\|\sum_{j=1}^K v_j\|_2^2 \le K\sum_{j=1}^K \|v_j\|_2^2$, $W_{t,j-1} \le 1$ and $\mathbb E[\mathbf 1\{k_t \ge j\} | \mathcal F_t] = \pi_{t,j}$,
\[
\mathbb E\bigl[\|g_t\|_2^2 | \mathcal F_t\bigr] \le 4F^2 K \sum_{j=1}^K \frac{1}{\pi_{t,j}} \le \frac{4F^2K^2}{\varepsilon}.
\]
Substituting $M^2 = 4F^2K^2/\varepsilon$ into~\eqref{eq:mtp-ogd-regret} gives the regression bound.
 
\emph{Selection error.} On a greedy round $\xi_t = 0$. On an exploration round, $\xi_t = \max_k \widehat R_t(k) - \widehat R_t(K) \le (K+1)/c_{\min}$, so $\mathbb E\sum_t \xi_t \le \varepsilon T(K+1)/c_{\min}$. Combining these bounds gives the final rates.
\end{proof}
 
\begin{proof}[Proof of Corollary~\ref{cor:speedup}]
Let $\Delta_t = R_t(k_t^\star) - R_t(k_t) \ge 0$. Since $R_t(k_t^\star) \ge \rho$ and $C_t(k_t) \le c_{\max}$,
\[
G_t(k_t) = C_t(k_t)\bigl(R_t(k_t^\star) - \Delta_t\bigr) \ge \rho\, C_t(k_t) - c_{\max}\Delta_t .
\]
Every round commits at least one token, so $\sigma_N \le N$, and the last round overshoots the budget by at most $K$ tokens: $N \le \sum_{t=1}^{\sigma_N} L_t \le N + K$. The event $\{\sigma_N \ge t\}$ is determined before round $t$, so taking conditional expectations round by round and summing gives
\[
\rho\,\mathbb E J_N
\le \mathbb E\sum_{t=1}^{\sigma_N} L_t + c_{\max}\,\mathbb E\sum_{t=1}^{\sigma_N}\Delta_t
\le N + K + c_{\max}\,\mathbb E\,\mathcal D_N ,
\]
where the last step uses $\sigma_N \le N$ and $\Delta_t \ge 0$. Dividing by $\rho$ gives the completion-time bound, and dividing the baseline's expected completion time by it gives the speedup bound.
\end{proof}
 
\textbf{The square-root dependence is tight.} The transfer in Theorem~\ref{thm:transfer} scales with the square root of the excess loss. The next result shows that this cannot be improved for the forecast-greedy controller, even with independent heads, known costs and $K = 2$.
 
\begin{proposition}[Square-root dependence is unavoidable]
\label{prop:mtp-sharpness}
For every $\delta \in (0, 1/2)$ there is an exact speculative-decoding instance with $K = 2$, costs $C(1) = 1$ and $C(2) = 5/4$, and a pre-draft hazard forecast with excess loss $E = 4\delta^2$, such that forecast-greedy selection incurs
\begin{equation}
R(k^\star) - R(\widehat k) = \frac{2}{5}\sqrt{E}.
\label{eq:mtp-sharpness-one-round}
\end{equation}
Repeating the construction over $T$ rounds gives
\begin{equation}
\mathcal D_T = \frac{2}{5}\sqrt{T\sum_{t=1}^T E_t}.
\label{eq:mtp-sharpness-cumulative}
\end{equation}
Hence no uniform $o(\sqrt E)$ one-round bound holds under the assumptions of the transfer theorem.
\end{proposition}
 
\begin{proof}
\emph{Instance.} Take the vocabulary $\{0,1\}$ and a target that emits $0$ with probability one after every prefix. The independent proposal heads are $q_1(0) = 1$ and $q_2(0) = 1/2 + \delta$. The first token is always accepted, so $W_1 = 1$ and $e_1 = h_1 = 0$. At the second position $e_2 = \mathrm{TV}(p_2, q_2) = 1/2 - \delta$: a proposed $0$ is always accepted and a proposed $1$ always rejected, and both the correction and the bonus token are $0$. Hence $h_2 = 1/2 - \delta$, $s_1 = 1$ and $s_2 = 1/2 + \delta$.
 
\emph{Forecast and excess loss.} Take $\widehat h_1 = 0$ and $\widehat h_2 = 1/2 + \delta$. The discrepancies are deterministic and $W_0 = W_1 = 1$, so $\ell(\widehat h) = (2\delta)^2 = 4\delta^2$ and $\ell^{\mathrm B} = 0$, giving $E = 4\delta^2$. Both the forecast and the true hazards lie in the affine class $\widehat h_{\theta,j} = \theta_j$ on $\Theta = [0,1]^2$.
 
\emph{Regret.} Counting the correction or bonus token, $G(1) = \widehat G(1) = 2$, $G(2) = 5/2 + \delta$ and $\widehat G(2) = 5/2 - \delta$. With the stated costs, $R(1) = \widehat R(1) = 2$, $R(2) = 2 + 4\delta/5$ and $\widehat R(2) = 2 - 4\delta/5$. So $k^\star = 2$ and $\widehat k = 1$ are unique, and the regret is $4\delta/5 = (2/5)\sqrt E$. Since $K$ and the costs stay fixed as $\delta \downarrow 0$, no bound of the form $o(\sqrt E)$, and in particular none of the form $cE$, can hold uniformly.
 
\emph{Many rounds.} Use the same models, costs and forecasts at every round. All committed tokens are $0$ and the models ignore the history, so these are consistent rounds of one decoding process, with $E_t = 4\delta^2$ and $\mathcal D_T = 4T\delta/5 = (2/5)\sqrt{T\sum_t E_t}$.
\end{proof}

The construction above places the two horizons arbitrarily close to a decision boundary. When every suboptimal horizon is separated from the best one, the transfer becomes linear in the excess loss.
 
\begin{proposition}[Transfer under a throughput gap]
\label{prop:mtp-margin}
Under the known-cost assumptions of Theorem~\ref{thm:transfer}, suppose $\exists\gamma > 0$ such that $R_t(k_t^\star) - R_t(k) \in \{0\} \cup [\gamma, \infty)$ for every round $t$ and horizon $k$. Then
\[
\mathcal D_T \le \frac{4\Lambda_K^2}{\gamma}\bigl(\mathcal R_T + \mathcal A_T\bigr) + 2\sum_{t=1}^T \xi_t .
\]
\end{proposition}
 
\begin{proof}
Let $\Delta_t = R_t(k_t^\star) - R_t(k_t)$ and $d_t = \Lambda_K\sqrt{E_t}$. Step 2 of the proof of Theorem~\ref{thm:transfer} gives $\Delta_t \le 2d_t + \xi_t$. If $\Delta_t = 0$ there is nothing to prove. Otherwise $\Delta_t \ge \gamma$, and Young's inequality $2d_t \le \Delta_t/2 + 2d_t^2/\Delta_t$ gives, after rearranging,
\[
\Delta_t \le \frac{4d_t^2}{\Delta_t} + 2\xi_t \le \frac{4\Lambda_K^2}{\gamma}E_t + 2\xi_t .
\]
Summing over $t$ completes the proof.
\end{proof}
 
\textbf{Predicting survival directly.} Hazards are one way to parameterize acceptance survival. The same verification feedback also supports predicting $s_{t,j}$ directly, using either of the labels
\begin{equation}
Z_{t,j} = W_{t,j},
\qquad
Z_{t,j}^{\mathrm{RB}} = W_{t,j-1}\bigl(1 - e^{\mathrm{obs}}_{t,j}\bigr).
\label{eq:direct-survival-labels}
\end{equation}
Both have conditional mean $s_{t,j}$ given $\mathcal F_t$. The second equals $\mathbb E[W_{t,j} | \mathcal F_t, Y_{<j}]$: it integrates out the current proposed token, so it is a Rao--Blackwellization of the first and has no larger conditional variance. Both are available at every drafted position under the feedback assumptions of Section~\ref{sec:controller}.
 
Let $\widehat s_{\psi,t,j}$ be a pre-draft survival forecast, let $\zeta_{t,j}$ be either label, and define $\ell_t^{\mathrm S}(\psi) = \mathbb E_{Q_t}\bigl[\sum_{j=1}^K (\zeta_{t,j} - \widehat s_{\psi,t,j})^2\bigr]$. As in Lemma~\ref{lem:mtp-calibration}, its excess loss is
\begin{equation}
E_t^{\mathrm S}(\psi) := \ell_t^{\mathrm S}(\psi) - \ell_t^{\mathrm S,B} = \sum_{j=1}^K \bigl(\widehat s_{\psi,t,j} - s_{t,j}\bigr)^2 .
\label{eq:direct-survival-calibration}
\end{equation}
With the same exploration scheme, the observed loss $\widehat\ell_t^{\mathrm S}(\psi) = \sum_{j=1}^{k_t} (\zeta_{t,j} - \widehat s_{\psi,t,j})^2 / \pi_{t,j}$ is conditionally unbiased. With $\widehat G_t^{\mathrm S}(k) = 1 + \sum_{j=1}^k \widehat s_{\psi_t,t,j}$, Cauchy-Schwarz gives
\begin{equation}
\bigl|\widehat G_t^{\mathrm S}(k) - G_t(k)\bigr| \le \sqrt{k}\,\sqrt{E_t^{\mathrm S}(\psi_t)},
\label{eq:direct-survival-gain}
\end{equation}
so with known costs the one-round throughput regret is at most $2\sqrt{K E_t^{\mathrm S}(\psi_t)}/c_{\min} + \xi_t$.

\section{Round Costs and Measured Throughput}
\label{app:costs}
The marginal value of an extra position makes the tradeoff explicit. If $\Delta C_t(k)=C_t(k+1)-C_t(k)\ge0$, then
\begin{equation}
 R_t(k+1)\ge R_t(k)
 \quad\Longleftrightarrow\quad
 s_{t,k+1}\ge R_t(k)\,\Delta C_t(k).
 \label{eq:marginal-rule}
\end{equation}
The gain in expected committed tokens must cover the extra round cost. This comparison is local; the controller scores every candidate and does not assume that $R_t(k)$ is unimodal.

\paragraph{Estimated latency costs.} Write $E_t=E_t(\theta_t)$ as in \eqref{eq:mtp-calibration}.
If the controller instead uses
$\widehat R_t(k)=\widehat G_t(k)/\widehat C_t(k)$,
assume $\widehat C_t(k)\ge c_{\min}$ and define
\[
\delta_t^C
=
\max_{1\le k\le K}
|\widehat C_t(k)-C_t(k)|.
\]
Then
\[
\max_k|\widehat R_t(k)-R_t(k)|
\le
\Lambda_K\sqrt{E_t}
+
\frac{K+1}{c_{\min}^2}\delta_t^C.
\]
Thus Theorem~\ref{thm:transfer} becomes
\[
\mathcal D_T
\le
2\Lambda_K\sqrt{T(\mathcal R_T+\mathcal A_T)}
+
\frac{2(K+1)}{c_{\min}^2}
\sum_{t=1}^T\delta_t^C
+
\sum_{t=1}^T\xi_t,
\]
where $\xi_t$ is now defined relative to the estimated-cost
objective.

\paragraph{Connection to measured acceptance.} Over $T$ complete rounds, define the pooled draft acceptance
rate, average window length, and average realized round cost by
\[
\widehat a_T
=
\frac{\sum_{t=1}^T A_t(k_t)}
     {\sum_{t=1}^T k_t},
\qquad
\overline k_T
=
\frac1T\sum_{t=1}^T k_t,
\qquad
\overline c_T
=
\frac1T\sum_{t=1}^T c_t.
\]
Since $L_t=1+A_t(k_t)$, the measured throughput satisfies the
exact identity
\begin{equation}
\frac{\sum_{t=1}^T L_t}{\sum_{t=1}^T c_t}
=
\frac{1+\widehat a_T\,\overline k_T}{\overline c_T}.
\label{eq:mtp-empirical-throughput}
\end{equation}
Consequently, acceptance improvements translate into throughput
improvements through both the average window length and the
round cost.
Acceptance alone is insufficient: for a fixed context and draft
law,
\[
\frac{\mathbb E[A_t(k)\mid\mathcal F_t]}{k}
=
\frac1k\sum_{j=1}^k s_{t,j}
\]
is nonincreasing in $k$, so shortening the window increases
acceptance without increasing throughput.

The identity in \eqref{eq:mtp-empirical-throughput} applies to complete rounds under the one-correction-or-bonus convention. An output cutoff or early end-of-sequence token requires recording the actual committed count; the pooled identity should not be applied blindly to such truncated rounds.

\newpage
\section{Experimental Setup and Further Results}
\label{app:setup}
This section describes our experimental setup and provides further results. We use the following checkpoints that can be found on HuggingFace:

\begin{enumerate}
\item \texttt{jwkirchenbauer/Qwen3-4B-Inst-2507-MTP}, 
\item \texttt{facebook/multi-token-prediction} (7B\_1T\_4), 
\item \texttt{meta-llama/Llama-3.1-8B-Instruct} with Medusa heads, 
\item \texttt{Qwen/Qwen3-14B}, 
\item \texttt{meta-llama/Llama-2-\{7b,13b\}-hf}.
\end{enumerate}

\begin{table*}[hbt!]
\centering
\footnotesize
\caption{MTP draft/target experiment matrix, run without KV-caches. Rows 1-3 are the core self-speculative pairings; rows 4-5 are the cross-model scaling pairings. Each row is evaluated on GSM8K and AIME 2025 with a fixed window, AMTP, and plain (next-token) decoding of the target.}
\label{tab:mtp-experiment-matrix}
\resizebox{\textwidth}{!}{
\begin{tabular}{@{}clllllc@{}}
\toprule
\# & Family & Pairing & Draft & Target & Plain baseline & Max $k$ \\
\midrule
1 & Self-distilled & Self-speculative & Qwen3-4B-Inst-2507-MTP & Same checkpoint & Same checkpoint & 16 \\
2 & Meta MTP       & Self-speculative & multi-token-prediction 7B & Same checkpoint & Llama-2-7B & 4 \\
3 & Medusa         & Self-speculative & Llama-3.1-8B-Instruct + Medusa heads & Same checkpoint & Llama-3.1-8B-Instruct & 6 \\
\midrule
4 & Self-distilled & Cross-model & Qwen3-4B-Inst-2507-MTP & Qwen3-14B & Qwen3-14B & 16 \\
5 & Meta MTP       & Cross-model & multi-token-prediction 7B & Llama-2-13B & Llama-2-13B & 4 \\
\bottomrule
\end{tabular}}
\end{table*}

\textbf{Hardware and Software.}
Self-distilled Qwen3-4B, its Qwen3-14B cross-model pairing and Medusa runs use a NVIDIA H100 on a university SLURM cluster; Meta MTP runs use a single NVIDIA H100 SXM on a rented cloud pod. Every run occupies one GPU exclusively and processes one sample at a time, so measured wall time is not distorted by co-located jobs. Throughput is therefore comparable between conditions on the same hardware but not across the two hardware groups. Models are loaded in the checkpoint's native precision through Hugging Face \texttt{transformers}; evaluations are driven by Inspect~AI, using a custom model provider that wraps our speculative-decoding engine. The exact package versions are pinned in the released code's lock file.\looseness=-1

\textbf{Tasks and Caps.} \textbf{AIME 2025} uses all 30 problems with a limit of 1{,}000 verification rounds and 3{,}000 generated tokens per problem. \textbf{GSM8K} uses the standard ten-shot prompt (fixed few-shot seed) with a limit of 300 rounds per problem for the speculative conditions. The self-distilled runs (self-speculative and paired with the Qwen3-14B target) use all 1{,}319 test problems; the Medusa runs use the first 250. The Meta MTP checkpoints are base models that do not emit an end-of-sequence token after answering, so they generate until the round cap; for the Meta 13B pair on GSM8K we therefore cap each of the first 100 problems at 100 rounds (300 tokens for plain decoding) and use identical caps across conditions, and we re-analyse the earlier 300-round Meta 7B fixed run on the same prefix. These runs measure throughput and acceptance only, since a 7--13B base model's GSM8K accuracy is not a meaningful check of lossless decoding. The Qwen3-14B plain-decoding baseline for the self-distilled cross-model pairing reached 1{,}166 of 1{,}319 GSM8K problems before its time limit; its throughput and accuracy are computed over that completed subset, and the fixed and adaptive conditions for this pairing still cover all 1{,}319. Sampling draws from the unmodified softmax of each model (no temperature scaling), without a fixed random seed, so runs are not bit-for-bit repeatable.

\textbf{Metrics.} For a run with rounds $t=1,\dots,T$, let $c_t$ be the number of tokens committed in round $t$ (accepted drafts plus the one token sampled from the target), $k_t$ the window used, and $\tau_t$ the measured wall time of the round, including drafting, verification and synchronization.
\begin{align*}
\text{tokens/round} &= \tfrac{1}{T}\textstyle\sum_t c_t, &
\text{acceptance rate} &= \tfrac{\sum_t (c_t-1)}{\sum_t k_t},\\
\text{throughput (tok/s)} &= \tfrac{\sum_t c_t}{\sum_t \tau_t}, &
\text{speed-up} &= \tfrac{\text{throughput}}{\text{throughput of plain decoding}}.
\end{align*}
For plain next-token decoding, a round produces exactly one token ($c_t{=}1$, no drafting). Sums run over all problems of an evaluation, so throughput is a token-weighted, not a per-problem, average. We report accuracy from the task's own scorer.

\textbf{Baselines.} Plain decoding uses the same target model and prompt with the standard autoregressive loop, also without a KV-cache, so speed-ups isolate the effect of speculation. For the self-speculative pairings (self-distilled Qwen3-4B, Medusa on Llama-3.1-8B-Instruct) the plain-decoding baseline is the same checkpoint decoding without drafting. Meta's 7B multi-token-prediction checkpoint has no separately released plain next-token model, so its plain-decoding baseline is Llama-2-7B, the base architecture it was fine-tuned from; for the 13B pairing, the target itself, Llama-2-13B, is the baseline. For the self-distilled cross-model pairing, the baseline is the target alone, Qwen3-14B.

\subsection{Draft-Model Implementations}\label{app:drafts}

All draft models implement one interface: given the current context, return $k$ sampled tokens together with the $k$ distributions they were sampled from. The engine is agnostic to how the draft produces them, so AMTP and the verification step are shared across families.

\textbf{Self-speculative and cross-model pairings.} In a self-speculative pairing the same weights draft and verify and are loaded once. In a cross-model pairing a small multi-token-prediction draft proposes tokens for a larger target of the same vocabulary: self-distilled Qwen3-4B for Qwen3-14B, and Meta's 7B multi-token-prediction model for Llama-2-13B. The engine checks that the draft and target share a tokenizer and vocabulary; the self-distilled model declares extra mask tokens.

\textbf{Self-distilled Qwen3-4B (mask-token drafting).} The checkpoint drafts $k$ tokens in one forward pass by appending $k{-}1$ reserved mask tokens (ids 151669--151700) to the context and reading the logits at the last real position and at each mask position. It supports $k\le 16$. It ships without its own verification scheme, so verification is performed by our engine with the standard accept/reject rule.

\textbf{Meta MTP (independent parallel heads).} The 7B checkpoint has four output heads that read the trunk's final hidden state and predict tokens $t{+}1,\dots,t{+}4$ independently. The checkpoint is not in Hugging Face format, so we load it with a self-contained implementation (transformer trunk, four heads, SentencePiece tokenizer). A single trunk forward produces all 4 distributions.

\textbf{Medusa on Llama-3.1-8B-Instruct.} Six residual-block heads, each with its own output projection, predict from the last hidden state of the base model. The base model's own next-token distribution supplies position one; head $i$ then predicts position $i{+}2$. Head weights are loaded from a separate repository, and the base model is the unmodified instruct checkpoint.

\textbf{No KV-cache.} All runs recompute the full context on every forward pass. This makes per-round cost depend visibly on context length, which AMTP's cost model represents (Appendix~\ref{app:cb-detail}), and it removes cache-rollback logic that differs between attention variants and would confound the comparison across draft families. Absolute throughput is consequently lower than with an optimised cached implementation; the relative comparisons between conditions are what we report.

\begin{table}[t]
\centering
\small
\begin{tabular}{@{}lr@{}}
\toprule
Quantity & Value \\
\midrule
Maximum window $K$ & 4 / 6 / 16 (Meta / Medusa / self-dist.) \\
Hazard model & affine (linear, no sigmoid link), 5 inputs \\
Hazard optimiser, learning rate & Projected OGD, $0.01$ \\
Hazard projection radius $D$ & $10.0$ \\
Exploration rate $\varepsilon$ & $0.1$ \\
Cost optimiser, learning rate & Adam, $0.01$ \\
Start-up updates $u_{\min}$ & $20$ \\
Context scale $\ell$, cost scale $\gamma$ & $100$, $1000$ \\
Checkpoint interval & every $50$ updates \\
\bottomrule
\end{tabular}
\caption{AMTP configuration used in all experiments. No value was tuned per model or per task.}
\label{tab:cb-config}
\end{table}

\subsection{Adaptive Multi-Token Prediction Details}\label{app:cb-detail}

\textbf{Hazard targets.} The hazard label at drafted position $j$ is $\mathrm{TV}(p_j,q_j)=1-\sum_v\min(p_j(v),q_j(v))$, the probability that a token drawn from $p_j$ is rejected. It is computed for every drafted position, including those after the first rejection, via the target distributions already produced by the verification.\looseness=-1

\textbf{Cost target.} The cost model is fitted to the whole round's measured wall time, as AMTP's decision only needs the total. Since every run disables the KV-cache (Appendix~\ref{app:drafts}), the cost model includes context length as a feature alongside the window $k$; a model fit with context length excluded explains almost none of the round-to-round cost variance under this setting.

\textbf{Checkpoint lineage.} Control-block state is saved every 50 updates under a key made of the draft checkpoint, target checkpoint and an implementation tag. The tag changes whenever the cost regime changes, and the current tag identifies runs without a key--value cache and with the rescaled cost model, so weights fitted under a different regime are never loaded.

\textbf{Start-up.} For the first 20 updates the window is drawn with probability proportional to $1/(n_k+1)$, so the cost model observes several window sizes. All conditions labelled ``adaptive'' begin with this phase; the numbers we report include it.

\subsection{Complete Results}\label{app:results}

Tables~\ref{tab:full-aime} and~\ref{tab:full-gsm8k} list our final results behind the main-text tables, and Table~\ref{tab:k-dist} gives the distribution of windows chosen by AMTP. Tables~\ref{tab:throughput}--\ref{tab:ktstats} in the main text are derived from the same records.

\pgfplotsset{
  paper/.style={tick label style={font=\footnotesize}, label style={font=\footnotesize}, title style={font=\small},
    legend style={font=\scriptsize, draw=none, fill=none}, ymajorgrids, grid style={gray!25}, axis line style={black!70}},
}

\begin{table}[hbt!]
\centering
\footnotesize
\setlength{\tabcolsep}{3pt}
\begin{tabular}{@{}lrrrrrr@{}}
\toprule
Model / eval & $k{=}1$ & $2$ & $3$ & $4$ & $5$--$6$ & $\geq 7$ \\
\midrule
Meta MTP 7B, AIME & 1.4 & 7.3 & 27.0 & 64.3 & 0.0 & 0.0 \\
Meta MTP 13B, AIME & 9.5 & 24.6 & 29.5 & 36.5 & 0.0 & 0.0 \\
Meta MTP 7B, GSM8K\textsuperscript{$\ast$} & 0.3 & 1.4 & 19.0 & 79.2 & 0.0 & 0.0 \\
Meta MTP 13B, GSM8K\textsuperscript{$\ast$} & 0.5 & 2.9 & 16.7 & 79.9 & 0.0 & 0.0 \\
Medusa 8B, GSM8K & 0.1 & 4.7 & 29.2 & 43.5 & 22.5 & 0.0 \\
Self-dist.\ 4B, GSM8K & 0.0 & 0.0 & 0.0 & 0.2 & 2.9 & 96.9 \\
Self-dist.\ 4B $\to$ 14B, AIME & 0.3 & 1.1 & 4.7 & 10.4 & 29.7 & 53.8 \\
Self-dist.\ 4B $\to$ 14B, GSM8K & 0.0 & 0.0 & 0.2 & 0.6 & 4.0 & 95.1 \\
\bottomrule
\end{tabular}
\caption{Share of verification rounds (\%) in which the control block chose each window size, over the full round population. The distribution is wide for every model, and the mode differs by model and by task. $^\ast$Capped runs (first 100 problems, first 100 rounds).}
\label{tab:k-dist}
\end{table}

\begin{table*}[hbt!]
\centering
\footnotesize
\setlength{\tabcolsep}{3pt}
\begin{tabular}{@{}llrrrrrr@{}}
\toprule
Draft $\to$ target & Condition & Rounds & Tok./round & Accept.\ (\%) & Tok./s & Gen.\ (h) & Acc.\ (\%) \\
\midrule
Self-dist.\ Qwen3-4B & Plain & 81{,}791 & 1.00 & -- & 18.8 & 1.21 & 20.0 \\
 & Fixed $k{=}16$ & 25{,}549 & 3.84 & 17.7 & 34.0 & 0.80 & 30.0 \\
 & Adaptive & 25{,}927 & 3.83 & 33.9 & 34.5 & 0.80 & 26.7 \\
\midrule
Self-dist.\ 4B $\to$ Qwen3-14B & Qwen3-14B & 60{,}091 & 1.00 & -- & 10.4 & 1.61 & 20.0 \\
 & Fixed $k{=}16$ & 18{,}244 & 3.23 & 14.0 & 22.5 & 0.73 & 20.0 \\
 & Adaptive & 16{,}525 & 3.20 & 28.2 & 24.7 & 0.60 & 23.3 \\
\midrule
Meta MTP 7B & Llama-2-7B & 13{,}190 & 1.00 & -- & 12.5 & 0.29 & 0.0 \\
 & Fixed $k{=}4$ & 7{,}009 & 3.00 & 50.0 & 37.7 & 0.16 & 0.0 \\
 & Adaptive & 11{,}268 & 3.03 & 57.2 & 31.9 & 0.30 & 0.0 \\
\midrule
Meta MTP 7B $\to$ Llama-2-13B & Llama-2-13B & 14{,}309 & 1.00 & -- & 6.2 & 0.64 & 0.0 \\
 & Fixed $k{=}4$ & 5{,}370 & 2.76 & 43.9 & 7.6 & 0.54 & 0.0 \\
 & Adaptive & 2{,}772 & 2.51 & 51.6 & 28.2 & 0.07 & 0.0 \\
\midrule
Medusa Llama-3.1-8B & Plain & 53{,}294 & 1.00 & -- & 7.5 & 1.97 & 0.0 \\
 & Fixed $k{=}6$ & 19{,}760 & 2.34 & 22.4 & 23.6 & 0.55 & 0.0 \\
 & Adaptive & 18{,}395 & 2.28 & 46.2 & 25.3 & 0.46 & 0.0 \\
\bottomrule
\end{tabular}
\caption{AIME 2025 evaluations ($n{=}30$ problems in each row). Rounds: total verification rounds over all problems (one round is one token for plain decoding). Tok./s: total committed tokens per total measured round wall time. Gen.\ (h): summed round wall time. Acc.: fraction of problems answered correctly; it is $0$ for every 7--13B model (Meta and Llama), so it carries no signal there.}
\label{tab:full-aime}
\end{table*}

\begin{table*}[hbt!]
\centering
\footnotesize
\setlength{\tabcolsep}{3pt}
\resizebox{\textwidth}{!}{%
\begin{tabular}{@{}llrrrrrrr@{}}
\toprule
Draft $\to$ target & Condition & Problems & Rounds & Tok./round & Accept.\ (\%) & Tok./s & Gen.\ (h) & Acc.\ (\%) \\
\midrule
Self-dist.\ Qwen3-4B & Plain & 1319 & 279{,}050 & 1.00 & -- & 16.2 & 4.80 & 89.4 \\
 & Fixed $k{=}16$ & 1319 & 64{,}008 & 4.67 & 23.0 & 38.6 & 2.15 & 85.0 \\
 & Adaptive & 1319 & 65{,}183 & 4.63 & 32.4 & 38.3 & 2.19 & 85.7 \\
\midrule
Self-dist.\ 4B $\to$ Qwen3-14B & Qwen3-14B\textsuperscript{$\dagger$} & 1166 & 179{,}720 & 1.00 & -- & 6.4 & 7.80 & 92.8 \\
 & Fixed $k{=}16$ & 1319 & 47{,}643 & 4.33 & 20.8 & 21.3 & 2.69 & 92.9 \\
 & Adaptive & 1319 & 47{,}529 & 4.33 & 28.5 & 21.5 & 2.66 & 93.5 \\
\midrule
Meta MTP 7B\textsuperscript{$\ast$} & Llama-2-7B & 100 & 24{,}440 & 1.00 & -- & 4.7 & 1.44 & -- \\
 & Fixed $k{=}4$ & 100 & 8{,}913 & 3.18 & 54.5 & 19.2 & 0.41 & -- \\
 & Adaptive & 100 & 9{,}065 & 3.21 & 58.6 & 19.4 & 0.42 & -- \\
\midrule
Meta MTP 7B $\to$ Llama-2-13B\textsuperscript{$\ast$} & Llama-2-13B & 100 & 23{,}660 & 1.00 & -- & 2.7 & 2.46 & -- \\
 & Fixed $k{=}4$ & 100 & 8{,}108 & 2.72 & 42.9 & 4.7 & 1.31 & -- \\
 & Adaptive & 100 & 8{,}206 & 2.75 & 46.6 & 11.6 & 0.54 & -- \\
\midrule
Medusa Llama-3.1-8B & Plain & 250 & 44{,}877 & 1.00 & -- & 13.0 & 0.96 & 83.2 \\
 & Fixed $k{=}6$ & 250 & 15{,}619 & 2.99 & 33.1 & 20.0 & 0.65 & 79.6 \\
 & Adaptive & 250 & 16{,}739 & 2.90 & 48.2 & 19.4 & 0.69 & 79.2 \\
\bottomrule
\end{tabular}%
}
\caption{Complete GSM8K results. Self-distilled 4B and Qwen3-14B target rows use all 1{,}319 problems where finished; Medusa uses the first 250 problems; $\ast$Meta MTP runs are capped (first 100 problems, first 100 rounds; $\sim$300 tokens for plain decoding) because the Meta base models do not terminate after answering, so their scored accuracy (1--4\%, driven almost entirely by truncation) is not meaningful and is omitted. $\dagger$The Qwen3-14B plain run reached 1{,}166 of 1{,}319 problems before its time limit; its accuracy is over those completed problems.}
\label{tab:full-gsm8k}
\end{table*}

\begin{figure*}[hbt!]
\centering
\begin{subfigure}[t]{\textwidth}
\centering
\begin{tikzpicture}
\begin{groupplot}[
  group style={group size=2 by 1, horizontal sep=1.6cm},
  paper, ybar, width=0.52\textwidth, height=5.2cm,
  ymin=0, ymax=40, ylabel={Tokens / second},
  symbolic x coords={c0,c1,c2,c3,c4}, xtick=data,
  xticklabels={{Meta 7B\\(self)},{Meta 7B to 13B\\(cross)},{Medusa\\8B},{Self-dist.\\4B (self)},{Self-dist.\\4B to 14B (cross)}},
  xticklabel style={align=center, font=\tiny}, xticklabel pos=lower, enlarge x limits=0.15,
  legend to name={fig:throughput-legend},
  legend style={draw=none, font=\footnotesize, legend columns=-1, /tikz/every even column/.append style={column sep=0.4cm}}]
\nextgroupplot[title={AIME 2025}]
\addplot[fill=cPlain, draw=black!60, line width=0.3pt] coordinates {(c0,12.53)(c1,6.19)(c2,17.32)(c3,18.79)(c4,10.36)};
\addplot[fill=cFixed, draw=black!60, line width=0.3pt] coordinates {(c0,37.68)(c1,7.65)(c2,23.56)(c3,33.97)(c4,22.47)};
\addplot[fill=cAdapt, draw=black!60, line width=0.3pt] coordinates {(c0,31.93)(c1,28.17)(c2,25.31)(c3,34.47)(c4,24.67)};
\legend{Plain decoding, Best fixed $k$, Adaptive}
\nextgroupplot[title={GSM8K}]
\addplot[fill=cPlain, draw=black!60, line width=0.3pt] coordinates {(c0,4.71)(c1,2.68)(c2,12.92)(c3,16.16)(c4,6.40)};
\addplot[fill=cFixed, draw=black!60, line width=0.3pt] coordinates {(c0,19.22)(c1,4.68)(c2,19.99)(c3,38.61)(c4,21.31)};
\addplot[fill=cAdapt, draw=black!60, line width=0.3pt] coordinates {(c0,19.41)(c1,11.58)(c2,19.44)(c3,38.33)(c4,21.47)};
\end{groupplot}
\node at ($(group c1r1.north)!0.5!(group c2r1.north) + (0,1.1cm)$) {\ref{fig:throughput}};
\end{tikzpicture}
\caption{Decoding throughput (tokens/second) for plain decoding, the best fixed window, and AMTP, across all five pairings and both evaluation tasks. Every multi-token-prediction condition beats plain decoding; the Meta MTP cross-model pairing (facebook/multi-token-prediction draft to Llama-2-13B target) shows the largest absolute gap between fixed and adaptive.}
\label{fig:throughput}
\end{subfigure}

\vspace{0.6cm}
\begin{subfigure}[t]{\textwidth}
\centering
\begin{tikzpicture}
\begin{axis}[
  paper, ybar, width=0.97\textwidth, height=4.6cm, bar width=18pt,
  ymin=-24, ymax=54, ytick={-20,0,20,40}, ylabel={Adaptive vs.\ best fixed $k$ (\%)},
  symbolic x coords={c0,c1,c2,c3,c4,c5,c6,c7}, xtick=data,
  xticklabels={{Self-dist.\\4B AIME},{Self-dist.\\4B GSM8K},{Medusa\\AIME},{Medusa\\GSM8K},{Meta 7B\\GSM8K},{4B to 14B\\AIME},{4B to 14B\\GSM8K}},
  xticklabel style={align=center, font=\tiny}, xticklabel pos=lower, enlarge x limits=0.08,
  extra y ticks={0}, extra y tick labels={}, extra y tick style={grid=major, major grid style={black!60, thin}},
  nodes near coords, nodes near coords style={font=\scriptsize}, point meta=explicit symbolic]
\addplot[fill=cAdapt, draw=black!60, line width=0.3pt] coordinates {
  (c1,1.5) [+1\%]
  (c2,-0.7) [-1\%]
  (c3,7.4) [+7\%]
  (c4,-2.8) [-3\%]
  (c5,1.0) [+1\%]
  (c6,9.8) [+10\%]
  (c7,0.8) [+1\%]};
\end{axis}
\end{tikzpicture}
\caption{Throughput of AMTP, relative to the best fixed window, self-speculative pairings and the modest-effect cross-model pairing (self-distilled 4B$\to$14B).}
\label{fig:gain}
\end{subfigure}

\vspace{0.4cm}
\begin{subfigure}[t]{\textwidth}
\centering
\begin{tikzpicture}
\begin{axis}[
  paper, ybar, width=0.5\textwidth, height=4.2cm, bar width=32pt,
  ymin=0, ymax=330, ytick={0,100,200,300}, ylabel={Adaptive vs.\ best fixed $k$ (\%)},
  symbolic x coords={c0,c1}, xtick=data,
  xticklabels={{Meta MTP 7B to 13B\\AIME 2025},{Meta MTP 7B to 13B\\GSM8K}},
  xticklabel style={align=center, font=\small}, xticklabel pos=lower, enlarge x limits=0.4,
  nodes near coords, nodes near coords style={font=\small}, point meta=explicit symbolic]
\addplot[fill=cAdapt, draw=black!60, line width=0.3pt] coordinates {
  (c0,268.3) [+268\%]
  (c1,147.4) [+147\%]};
\end{axis}
\end{tikzpicture}
\caption{The cross-model Meta MTP pairing (draft/target cost asymmetry), shown on its own axis given the scale of the effect relative to every other condition.}
\label{fig:gain-crossmodel}
\end{subfigure}
\caption{Top: decoding throughput of plain decoding, the best fixed window and AMTP. Every multi-token-prediction condition is faster than plain decoding. Middle: throughput of AMTP relative to the best fixed window, for self-speculative pairings and the self-distilled cross-model pairing --- effects are modest in both directions, within about 15\% either way. Bottom: the Meta MTP cross-model pairing, shown separately because its gain ($+268\%$ AIME, $+147\%$ GSM8K) is an order of magnitude larger than every other condition --- draft/target cost asymmetry combined with long, variable-length reasoning produces the largest effect in this validation.}
\label{fig:gain-combined}
\end{figure*}
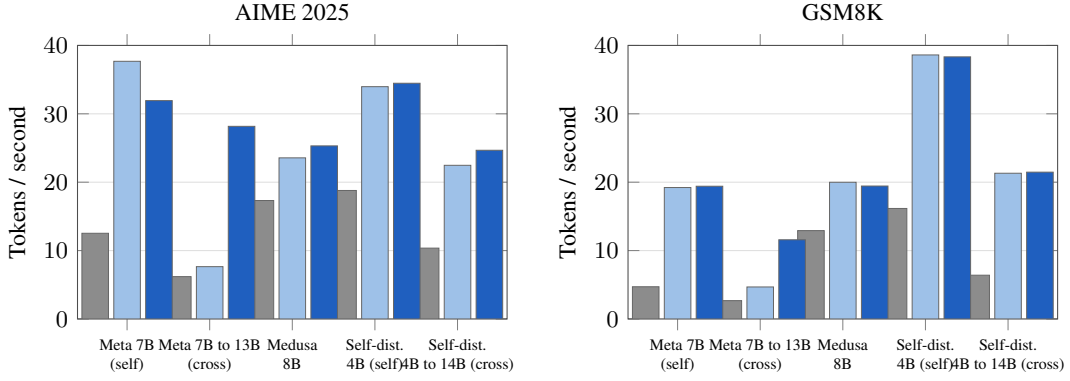
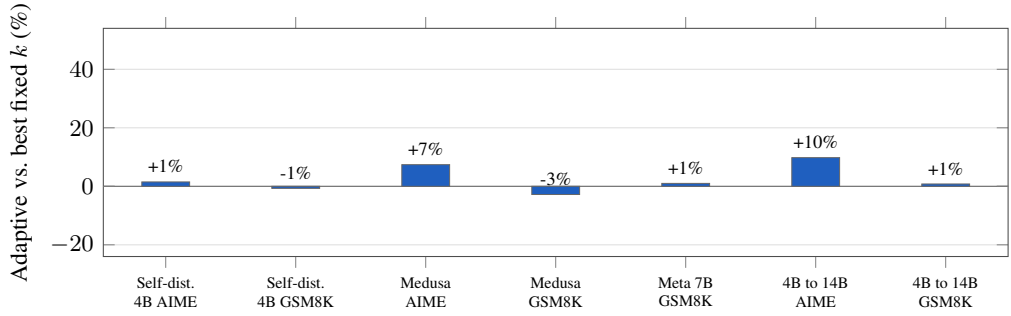
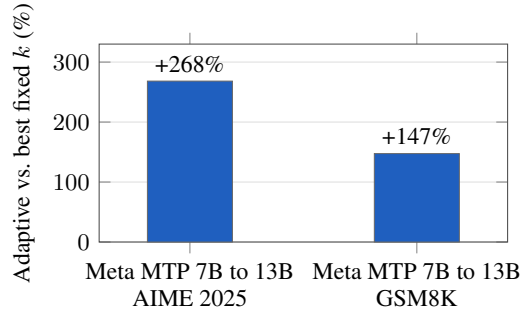

\textbf{Worked Examples.}
We provide real unedited excerpts from logged telemetry of a self-distilled Qwen3-4B adaptive run on AIME 2025: \emph{Let $A$ be the set of positive integer divisors of $2025$. Let $B$ be a randomly selected subset of $A$. The probability that $B$ is a nonempty set with the property that the least common multiple of its elements is $2025$ is $\frac{m}{n}$, where $m$ and $n$ are relatively prime positive integers. Find $m+n$.}

\begin{figure*}[hbt!]
\centering
\includegraphics[width=0.98\linewidth]{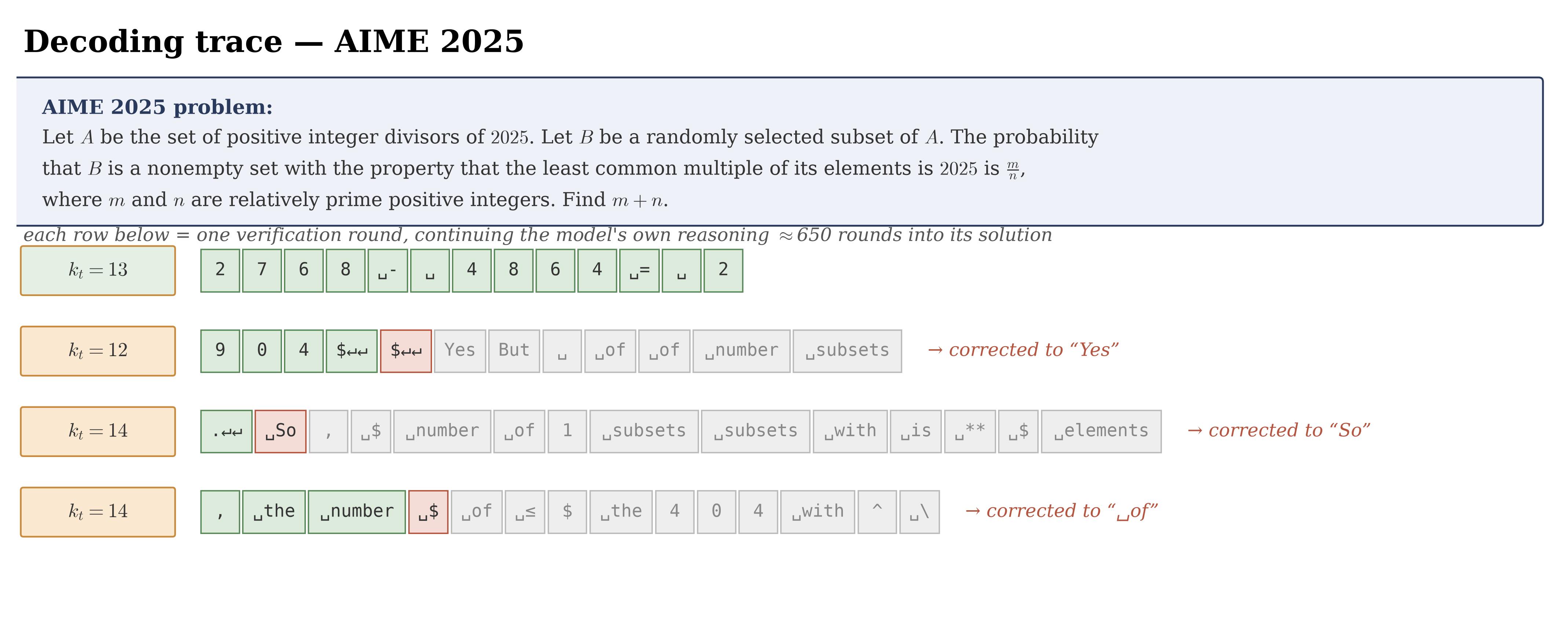}
\caption{Four consecutive real verification rounds (648--651) from one adaptive-control-block run on AIME 2025 (self-distilled Qwen3-4B). Green tokens were accepted, the red token was rejected and replaced by the correction shown on the right (resampled from the residual), and grey tokens were drafted but never verified because the round already stopped at the rejection. The control block re-selects the window $k_t$ every round (shown on the left). All tokens, windows, and corrections are taken directly from logged telemetry.}
\label{fig:decoding-trace}
\end{figure*}

\paragraph{Example 1: a difficult stretch, seven consecutive rounds.} Roughly 650 verification rounds into its own generated solution, the model is deriving the number of qualifying subsets. Table~\ref{tab:worked-trace} lists seven consecutive rounds verbatim. In each round the control block first chooses a window $k_t$ from its own hazard forecast, then drafts exactly $k_t$ tokens; the \emph{accepted} column gives how many of those were accepted before the first rejection (or all $k_t$, if none were rejected), and \emph{correction} gives the resampled token that replaced the first rejected draft token. Draft tokens past the rejection point are never verified and contribute nothing to the committed output.

\begin{table}[h]
\centering
\footnotesize
\setlength{\tabcolsep}{4pt}
\begin{tabular}{@{}rrp{5.6cm}rp{1.6cm}@{}}
\toprule
Round & $k_t$ & Draft & Accepted & Correction \\
\midrule
648 & 13 & \texttt{2 7 6 8 - ␣ 4 8 6 4 = ␣ 2} & 13/13 (all) & --- \\
649 & 12 & \texttt{9 0 4 \$ \$ Yes But ␣ ␣of ␣of ␣number ␣subsets} & 4/12 & \texttt{Yes} \\
650 & 14 & \texttt{. ␣So , ␣\$ ␣number ␣of 1 ␣subsets ␣subsets ␣with ␣is ␣** ␣\$ ␣elements} & 1/14 & \texttt{So} \\
651 & 14 & \texttt{, ␣the ␣number ␣\$ ␣of ␣$\le$ \$ ␣the 4 0 4 ␣with \textasciicircum{} ␣\textbackslash} & 3/14 & \texttt{␣of} \\
652 & 16 & \texttt{␣non ␣with B B ␣have ␣B (S ␣are ␣2 0 \textasciicircum{} \$ ␣= 2 2} & 1/16 & \texttt{empty} \\
653 & 8 & \texttt{␣subsets ␣\$ S ␣L CM ␣= ␣ 2} & 3/8 & \texttt{\$} \\
654 & 7 & \texttt{␣such ␣that ␣L ␣l S (S ␣} & 3/7 & \texttt{CM} \\
\bottomrule
\end{tabular}
\caption{Seven consecutive real verification rounds (648--654) from one adaptive-control-block run. \texttt{␣} marks a leading space in a token; \texttt{\$} and \texttt{\textbackslash} are literal draft tokens (the model was mid-LaTeX). All data is taken directly from logged round telemetry.}
\label{tab:worked-trace}
\end{table}

The committed text across these seven rounds, concatenated, reads: \emph{``2768 - 4864 = 27904\$\textbackslash n\textbackslash nYes.\textbackslash n\textbackslash nSo, the number of nonempty subsets \$S\$ such that LCM''} with two visible self-corrections (\texttt{Yes}, \texttt{So}) where the model's own verification caught a drafted continuation it would not have produced itself. The chosen window swings from $k_t{=}7$ to $k_t{=}16$ round to round, and every round in this stretch needed at least one correction; this is a difficult patch of the derivation, not representative of the easier, longer-window stretches that dominate most of a generation.

\paragraph{Example 2: full acceptance with a bonus token.} Later in the same run, round 101 is more typical of an easy stretch. The context ends ``\,\dots{}we want to count'', following directly from the probability setup. The control block forecasts rejection hazards $\widehat h_{t,1}{=}0.40$, $\widehat h_{t,2}{=}0.48$, $\widehat h_{t,3}{=}0.55$ for the next three positions, giving a predicted gain $\widehat G_t(3)\approx 1 + 0.60+0.31+0.14 \approx 2.05$ expected committed tokens, and the controller selects $k_t{=}3$. Table~\ref{tab:worked-trace-bonus} lists the single round. Since the draft is verified in full with no rejection, the engine draws one additional \emph{bonus} token directly from the target's own distribution rather than discarding the round.

\begin{table}[h]
\centering
\footnotesize
\setlength{\tabcolsep}{4pt}
\begin{tabular}{@{}rrp{3.2cm}rp{1.6cm}@{}}
\toprule
Round & $k_t$ & Draft & Accepted & Bonus \\
\midrule
101 & 3 & \texttt{␣the ␣number ␣of} & 3/3 (all) & \texttt{␣non} \\
\bottomrule
\end{tabular}
\caption{One real verification round (101) with full acceptance, from the same adaptive-control-block run as Table~\ref{tab:worked-trace}. The bonus token is sampled from the target's distribution at the position immediately after the draft, since no rejection occurred.}
\label{tab:worked-trace-bonus}
\end{table}

The committed text for the round is therefore ``\, the number of non'', four tokens from a single target forward pass, with no correction needed.

\end{document}